\UseRawInputEncoding
\documentclass{article}

\usepackage[numbers]{natbib}

\usepackage{microtype}
\usepackage{graphicx}
\usepackage{subfigure}
\usepackage{booktabs} 

\usepackage{hyperref}

\usepackage[most]{tcolorbox}   
\usepackage{lmodern}           

\newtcblisting{promptbox}{     
  colback = gray!5,
  colframe = black!75,
  fonttitle = \bfseries,
  title = Prompt,
  listing only,                
  listing options = {
    basicstyle   = \ttfamily\small,
    breaklines   = true,
    columns      = fullflexible,
    escapeinside = {(*@}{@*)}, 
  },
  sharp corners,
  enhanced,
  breakable                    
}

\newtcblisting{reasontrace}{
  listing only,
  listing options = {
    basicstyle       = \ttfamily\small,
    breaklines       = true,
    breakindent      = 0pt,     
    breakautoindent  = false,   
    columns          = fullflexible,
  },
  colback  = gray!7,
  colframe = gray!50,
  boxrule  = 0.25pt,
  arc      = 2pt,
  left = 2pt, right = 2pt, top = 1pt, bottom = 1pt,
  enhanced,
  breakable
}
  
\usepackage[accepted]{icml2025}

\usepackage{amsmath}
\usepackage{amssymb}
\usepackage{mathtools}
\usepackage{amsthm}

\usepackage[capitalize,noabbrev]{cleveref}

\usepackage[textsize=tiny]{todonotes}

\icmltitlerunning{Semantic Context for Tool Orchestration}

\DeclareMathOperator*{\argmax}{arg\,max}

\newcommand{\vect}[1]{\mathbf{#1}} 
\newcommand{\matr}[1]{\mathbf{#1}} 
\newcommand{\bm}[1]{\mathbf{#1}}

\newtheorem{theorem}{Theorem}[section]
\newtheorem{lemma}[theorem]{Lemma}

\newtheorem{definition}[theorem]{Definition}

\newtheorem{assumption}[theorem]{Assumption}

\begin{document}

\twocolumn[
\icmltitle{Semantic Context for Tool Orchestration}




\begin{icmlauthorlist}
\icmlauthor{Robert M\"uller}{yyy}
\end{icmlauthorlist}

\icmlaffiliation{yyy}{Aganthos}

\icmlcorrespondingauthor{Robert M\"uller}{aganthos.ai@gmail.com}

\icmlkeywords{Machine Learning, ICML}

\vskip 0.3in
]



\printAffiliationsAndNotice{}  

\begin{abstract}
This paper demonstrates that Semantic Context (SC), leveraging descriptive tool information, is a foundational component for robust tool orchestration. Our contributions are threefold. First, we provide a theoretical foundation using contextual bandits, introducing SC-LinUCB and proving it achieves lower regret and adapts favourably in dynamic action spaces. Second, we provide parallel empirical validation with Large Language Models, showing that SC is critical for successful in-context learning in both static (efficient learning) and non-stationary (robust adaptation) settings. Third, we propose the FiReAct pipeline, and demonstrate on a benchmark with over 10,000 tools that SC-based retrieval enables an LLM to effectively orchestrate over a large action space. These findings provide a comprehensive guide to building more sample-efficient, adaptive, and scalable orchestration agents.

\end{abstract}

\section{Introduction}
\label{sec:introduction_revised_v5}

The capacity of intelligent systems, particularly Large Language Models (LLMs), is significantly amplified by their ability to orchestrate external tools—such as APIs, auxiliary agents, or specialized functions \citep{Patil2023Gorilla, Qin2023ToolBench, Schick2023Toolformer}. This orchestration is a sequential decision-making task: given a user query and a dynamic tool catalogue, an agent must select and use the most appropriate tool. While reinforcement learning (RL) offers a principled framework, naive application (e.g., LLMs generating tool invocations token-by-token) creates intractably large action spaces ($V^L$ with vocabulary size $V$ and sequence length $L$), hindering learning. A common simplification presents the agent with an explicit list of $O$ available 
 indices or tools, $\mathcal{A}_{avail} = \{a_1, \dots, a_O\}$, from which to select. However, this often discards valuable semantic descriptions $D(\tau)$  associated with each tool (e.g., API doc strings, capability summaries). Recent works using RL to train LLM to orchestrate tools rely on the provision of tool names and descriptions in the prompts \citep{feng2025retoolreinforcementlearningstrategic, singh2025agenticreasoningtoolintegration, zhang2025nemotronresearchtooln1exploringtoolusinglanguage, qian2025toolrlrewardtoollearning}. \cite{lin2024hammerrobustfunctioncallingondevice} improve tool call reliability by random augmentation of tool and argument names, thus pushing the model to rely on tool descriptions. This paper investigates the critical and quantifiable advantages of equipping agents with what we term the \textit{Semantic Context (SC)}—the collection of semantic descriptions for all currently available actions.

This SC is not merely a helpful addition but a fundamental component for effective tool orchestration. Our work establishes this through three core findings. 

First, we provide a theoretical and empirical foundation showing that even in \textbf{static settings} with a fixed tool set, SC enables more efficient learning. To do this, we develop \textbf{SC-LinUCB}, a bandit algorithm, and prove that it achieves favourable regret compared with non-semantic baselines by creating a more parsimonious and accurate reward model (Section~\ref{sec:sc_linucb_theory_revised}). Empirical support is provided by SC-LINUCB and in-context learning experiments with LLM. 

Second, we demonstrate SC's critical role in \textbf{dynamic adaptation}. Our experiments show that as tools are added or removed, an agent leveraging SC adapts gracefully, whereas baselines suffer from catastrophic forgetting and require costly retraining. This highlights SC as a key enabler for continual learning in evolving environments for both, SC-LINUCB and in-context learning LLM.

Finally, we show how SC makes tool orchestration practical at scale through a \textbf{FiReAct (Filter-Reason-Act) pipeline}. We demonstrate that semantically filtering a large corpus of tools into a small, relevant set is \textit{essential} for maintaining high accuracy as the number of tools grows into the thousands. This scalable application bridges our theoretical insights with the practical challenges faced by modern LLM agents (SubSection~\ref{subsec:sc_scaling_action_space}).

Our research draws from the contextual bandit framework \citep{Langford2007TheEA, Chu2011Contextual}, with LinUCB \citep{AbbasiYadkori2011Improved} as a cornerstone, and contributes by rigorously analysing features from \textit{a priori semantic embeddings of natural language action descriptions} and quantifying their regret impact. While action representation learning from interaction is common in RL \citep{Chandak2019LearningAR, pathakota2023dctdualchanneltraining}, and using natural language for actions has been explored \citep{Tennenholtz2019NaturalLanguageActions}, our focus is on leveraging pre-existing, structured semantic information. Addressing dynamic action spaces, central to continual learning, we differ from \citet{Chandak2020Lifelong} who infer latent action structures; we demonstrate how \textit{explicit, given semantic descriptions} enable robust adaptation without relearning action space representations. This complements continual RL's focus on evolving reward/transition functions \citep{pmlr-v151-muller22a, Khetarpal2020Towards}, aiming to furnish principled insights for more sample-efficient, generalizable, and adaptive tool-orchestrating agents that explicitly leverage SC.

When dealing with high dimensional task-/ action spaces there is a variety of approaches to dial down complexity. Examples include learning action elimination networks\citep{zahavy2019learnlearnactionelimination} to approaches partitioning the task space based on task embeddings \citep{mueller2020taming}. More recent tool-RAG methods tackle the problem from a retrieval perspective: small LMs learn a function-mask head that suppresses irrelevant APIs at inference time \citep{lin2024hammerrobustfunctioncallingondevice}; completeness-oriented retrievers rank tools so that only a minimal yet sufficient subset is forwarded to the reasoner \citep{qu2024colt,shi2025toolret}. Hierarchical bandit strategies have been explored for scalable, human-in-the-loop tool retrieval~\citep{Wahed2023Marble}.
\section{Problem Formulation}
\label{sec:problem_formulation_revised_v5}

We model the task of selecting an appropriate tool for a given query as a \textbf{contextual bandit} problem. This framework allows us to rigorously analyse the decision-making that underpins tool orchestration.

At each discrete time step $t \in \{1, \dots, T\}$, an agent observes a context (a user query $q_t \in \mathcal{Q}$) and must select an action $a_t$ from a set of currently available tools, $\mathcal{A}_t = \{a_1, \dots, a_{O_t}\}$ of magnitude $O_t$. The environment is stochastic: for a given query $q_t$, each action $a_i \in \mathcal{A}_t$ has a true but unknown probability of success, $p^{\text{eff}}(a_i, q_t)$. After selecting $a_t$, the agent receives a stochastic binary reward $r_t \in \{0, 1\}$, drawn from a Bernoulli distribution governed by this probability: $r_t \sim \text{Bernoulli}(p^{\text{eff}}(a_t, q_t))$.

The agent's objective is to learn a policy $\pi(a_t | q_t, H_{t-1})$ that maximizes the cumulative reward (or Return), $\sum_{t=1}^{T} r_t$. This is equivalent to minimizing the \textbf{Cumulative Expected Regret}, defined as the sum of the per-step differences between the expected reward of the optimal action for a given query and the expected reward of the action the agent actually chose:
\begin{equation}
\label{eq:regret}
    R_T = \sum_{t=1}^{T} \left( \max_{a \in \mathcal{A}_t} p^{\text{eff}}(a, q_t) - p^{\text{eff}}(a_t, q_t) \right).
\end{equation}

The central hypothesis of this paper is that an agent's policy can learn more efficiently and adapt faster to changes in the action space $\mathcal{A}_t$ if it explicitly leverages the \textbf{Semantic Context}, the rich descriptions associated with each action, rather than treating actions as abstract, opaque indices.

\begin{definition}[Semantic Context, $C_S(\mathcal{A}_t)$]
\label{def:semantic_action_context}
Given $\mathcal{A}_t$, the set of available actions at time $t$, the \textbf{semantic context} $C_S(\mathcal{A}_t)$ is the collection of semantic information related to these actions. Specifically, $C_S(\mathcal{A}_t) = \{ (a_i, D(a_i)) \}_{a_i \in \mathcal{A}_t}$, where for each action $a_i$, $D(a_i)$ is its natural language description (e.g., docstring). Each description is mapped to a $d_{\text{emb}}$-dimensional \textbf{semantic context embedding} $\phi(a_i) = \Xi(D(a_i))$ via an embedding function $\Xi$. This provides structured, \textit{a priori} information about the available actions, allowing an agent's policy $\pi(a_t | s_t, C_S(\mathcal{A}_t))$ to leverage both the usual state $s_t$ (which includes the query $q_t$) and the SC.
\end{definition}

Our analysis begins with the simplest instantiation of this framework: a stationary environment where the set of actions is fixed.

\begin{definition}[Semantic Context Bandit, SC-Bandit]
\label{def:sc_bandit}
An \textbf{SC-Bandit} models a single-step decision with a \textit{static} action space $\mathcal{A}_{\text{avail}}$. At each step $t$, given a query $q_t$, the agent selects an action $a_t$ based on its policy $\pi(a_t | q_t, C_S(\mathcal{A}_{\text{avail}}))$, where the Semantic Context is fixed.
\end{definition}

For SC MDP \ref{def:sc_mdp_v5} and the Lifelong SC MDP \ref{def:lsc_mdp_v5} with non-stationary action space we refer to appendix \ref{app:subsec_sc_mdp}. In all frameworks, the central hypothesis is that explicit incorporation and effective utilization of the Semantic Action Context $C_S(\mathcal{A}_t)$ enable agents to achieve superior learning efficiency, generalization, and adaptability.
\section{Theoretical Framework: Semantic LinUCB}
\label{sec:sc_linucb_theory_revised}
We analyse Semantic Contextual Linear UCB (SC-LinUCB), an adaptation of the LinUCB algorithm \citep{AbbasiYadkori2011Improved} that leverages semantic information from action descriptions. Our analysis demonstrates that by incorporating well-structured semantic features, SC-LinUCB can achieve significantly lower regret than LinUCB variants relying on non-semantic action representations. This improvement stems from a more efficient representation of the underlying reward structure, leading to better generalization and reduced exploration complexity.


Our theoretical contribution focuses on how the specific construction of semantic features $\vect{x}^{(sem)}$ for SC-LinUCB leads to a more favorable instantiation of this generic bound compared to using non-semantic features.

\subsection{Contextual Linear Bandits and Feature Design}
\label{subsec:sclinucb_features}
We operate within the standard contextual linear bandit framework (detailed in Appendix \ref{app:formal_assumptions_sclinucb}). At each time step $t$, given a query (context) embedding $\vect{q}_t \in \mathbb{R}^{d_q}$, the agent selects a tool $\tau_j$ from the $K_t$ available tools. Each tool $\tau_j$ is associated with a semantic description embedding $\bm{\phi}_j \in \mathbb{R}^{d_{desc}}$. The expected reward $\mathbb{E}[R_t | \vect{x}_{t,j}] = \vect{x}_{t,j}^T \bm{\theta}^*$ is linear in the constructed $d$-dimensional feature vector $\vect{x}_{t,j}$. The SC-LinUCB algorithm itself is Algorithm \ref{alg:sc_linucb_appendix} (Appendix \ref{app:algo_details}).

The core of our analysis lies in comparing two feature construction strategies:
\begin{itemize}
    \item \textbf{SC-LinUCB Semantic Features ($\vect{x}^{(sem)}$):} We construct $\vect{x}^{(sem)}_{t,j} = [\vect{q}_t; \bm{\phi}_j; \text{sim}(\vect{q}_t, \bm{\phi}_j); 1]$. The resulting feature dimension is $d_{sem} = d_q + d_{desc} + 1 + 1$. This design explicitly incorporates query attributes, tool semantic attributes, and their direct alignment.
    \item \textbf{LinUCB-NS Non-Semantic Features ($\vect{x}^{(non-sem)}$):} As a baseline, we use features $\vect{x}^{(non-sem)}_{t,j} = [\vect{q}_t; \mathbf{e}_j; 1]$, where $\mathbf{e}_j \in \mathbb{R}^K$ is the one-hot encoding for tool $\tau_j$. The dimension is $d_{non-sem} = d_q + K + 1$. This baseline distinguishes tools by identity but lacks explicit shared semantic information.
\end{itemize}
The generic regret for LinUCB algorithms 
stated in Appendix \ref{app:std_lemmas_proofs}), scaling as $\mathcal{R}_T = \tilde{O}(d \cdot \sigma_{eff} \cdot \sqrt{T})$, where $d$ is the feature dimension and $\sigma_{eff}$ is the effective noise standard deviation (incorporating observation noise and linear model approximation error).

\subsection{Regret Advantage via Efficient Semantic Representation}

To formalize the advantage of $\vect{x}^{(sem)}$, we introduce an assumption about the nature of the true reward function.

\begin{assumption}[Semantically Structured Rewards]
\label{ass:semantically_structured_rewards}
The true expected reward function $f^*(\vect{q}, \bm{\phi})$ is primarily determined by a limited number of underlying semantic interaction patterns between queries and tool semantic properties. Specifically, there exists an optimal linear model in the semantic feature space, $(\vect{x}^{(sem)}_{t,j})^T \bm{\theta}_{sem}^*$, that approximates $f^*(\vect{q}_t, \bm{\phi}_j)$ with a mean squared error $\sigma_{approx,sem}^2$.
Further, to achieve a comparable or better linear approximation quality using non-semantic one-hot features, i.e., $(\vect{x}^{(non-sem)}_{t,j})^T \bm{\theta}_{non-sem}^* \approx f^*(\vect{q}_t, \bm{\phi}_j)$ with error $\sigma_{approx,non-sem}^2 \ge \sigma_{approx,sem}^2$, the dimensionality $d_{non-sem}$ (which scales with $K$) may be significantly larger than $d_{sem}$ if $K$ is large and there is semantic redundancy across tools (i.e., $d_{desc}+1 \ll K$).
\end{assumption}

\begin{theorem}[Regret Reduction for SC-LinUCB]
\label{thm:sc_linucb_adv_semantic_generalization}
Under Assumption \ref{app:formal_assumptions_sclinucb} (for both SC-LinUCB with $(d_{sem}, \sigma_{eff,sem}, S_{sem}, L_{sem})$ and LinUCB-NS with $(d_{non-sem}, \sigma_{eff,non-sem}, S_{non-sem}, L_{non-sem})$) and Assumption \ref{ass:semantically_structured_rewards}:
SC-LinUCB achieves a cumulative regret $\mathcal{R}_T(SC)$ that is less than or equal to the regret of LinUCB-NS, $\mathcal{R}_T(NS)$, if its semantic features lead to a more favorable combination of dimensionality and effective noise. Specifically, $\mathcal{R}_T(SC) \le \mathcal{R}_T(NS)$ if the factor $d_{sem} \cdot \sigma_{eff,sem}$ (ignoring constants and polylog terms from $\alpha$) is smaller than $d_{non-sem} \cdot \sigma_{eff,non-sem}$.
A significant improvement ($\mathcal{R}_T(SC) \ll \mathcal{R}_T(NS)$) is realized if:
\begin{enumerate}
    \item \textbf{Parsimonious Representation:} $d_{sem} \ll d_{non-sem}$ (achievable if $d_{desc}+1 \ll K$) while maintaining comparable or better approximation quality ($\sigma_{eff,sem} \lesssim \sigma_{eff,non-sem}$). The regret reduction factor is roughly $d_{sem}/d_{non-sem}$.
    \item \textbf{Superior Fit:} Even if $d_{sem} \approx d_{non-sem}$, if semantic features provide a substantially better linear approximation, then $\sigma_{eff,sem} \ll \sigma_{eff,non-sem}$, leading to a regret reduction factor of roughly $\sigma_{eff,sem}/\sigma_{eff,non-sem}$.
\end{enumerate}

\end{theorem}

\begin{proof}[Proof Sketch]
The underlying mechanism is that SC-LinUCB learns a single model $\hat{\bm{\theta}}_{sem}$ over features that encode shared semantic properties. This enables generalization: experience with one tool informs the valuation of other semantically similar tools for similar contexts, leading to more efficient exploration of the (context $\times$ action) space compared to LinUCB-NS, which learns tool-specific parameters more independently via orthogonal one-hot encodings.

The regret for LinUCB (Theorem 19.2 in \cite{Lattimore2020Bandit}) is $\mathcal{R}_T = \tilde{O}(\alpha d \sqrt{T})$, where the exploration parameter $\alpha \approx \sigma_{eff}\sqrt{d\log(TK/\delta)} + \text{const} \cdot S_\theta$. Thus, $\mathcal{R}_T \approx \tilde{O}((\sigma_{eff}d + S_\theta\sqrt{d}) \sqrt{T})$. For simplicity in comparing dominant effects, we consider the scaling $\tilde{O}(d \cdot \sigma_{eff} \cdot \sqrt{T})$.

\textbf{Case 1: Parsimony ($d_{sem} \ll d_{non-sem}$ with $\sigma_{eff,sem} \approx \sigma_{eff,non-sem} = \sigma_{eff}$).}
Assumption \ref{ass:semantically_structured_rewards} posits that relevant semantic information can be captured in $d_{sem} = d_q+d_{desc}+2$ dimensions. If $K$ tools share significant semantic redundancy (e.g., many tools belong to a few archetypes effectively distinguished by $d_{desc}$ features, so $d_{desc} \ll K$), then $d_{sem}$ does not need to scale with $K$. In contrast, LinUCB-NS requires $d_{non-sem} = d_q+K+1$ to assign distinct parameters to each tool via its one-hot encoding $\mathbf{e}_j$. If $d_{desc}+1 \ll K$, then $d_{sem} \ll d_{non-sem}$. Since the regret scales with $d$, SC-LinUCB directly benefits from operating in a lower-dimensional parameter space, assuming the quality of linear fit ($\sigma_{eff}$) is maintained. This reduction in $d$ means SC-LinUCB needs to estimate fewer parameters, and the confidence ellipsoids effectively cover the parameter space with fewer samples. The sum of exploration terms $\sum s_{t,\tau_t}^2$ is bounded by $O(d_{sem}\log T)$ for SC-LinUCB versus $O(d_{non-sem}\log T)$ for LinUCB-NS (Lemma \ref{lemma:appendix_elliptical_potential}).

\textbf{Case 2: Superior Fit ($\sigma_{eff,sem} < \sigma_{eff,non-sem}$ with $d_{sem} \approx d_{non-sem}$).}
If the true reward function $f^*(\vect{q}, \bm{\phi})$ is more closely aligned with a linear model over semantic features $\vect{x}^{(sem)}$ than over non-semantic features $\vect{x}^{(non-sem)}$, then the approximation error component of $\sigma_{eff,sem}$ will be smaller than that of $\sigma_{eff,non-sem}$. Semantic features like direct query-tool similarity $(\vect{q}_t^T \bm{\phi}_j)$ can capture critical interaction effects that might be poorly represented by the interactions of $\vect{q}_t$ with orthogonal one-hot vectors $\mathbf{e}_j$ in a linear model. A smaller $\sigma_{eff,sem}$ makes the exploration parameter $\alpha_{sem}$ smaller for SC-LinUCB (for the same confidence $\delta$), leading to tighter UCB scores $p_{t,j}$, thus less exploration of suboptimal arms. This means the term $2\alpha s_{t,\tau_t}$ in the instantaneous regret bound is smaller on average.

\textbf{Generalization Impact on Exploration Dynamics:}
Beyond the direct impact on $d$ and $\sigma_{eff}$, the shared nature of $\hat{\bm{\theta}}_{sem}$ in SC-LinUCB means that an update from exploring $(\vect{q}_t, \bm{\phi}_a)$ refines weights associated with semantic components (e.g., specific dimensions of $\bm{\phi}_a$ or the similarity feature). If another tool $\bm{\phi}_b$ shares these semantic components relevant to $\vect{q}_t$, the UCB score for $(\vect{q}_t, \bm{\phi}_b)$ is implicitly updated and its uncertainty reduced more effectively than in LinUCB-NS, where the learning for tool $a$ (via $\mathbf{e}_a$) is largely isolated from tool $b$ (via $\mathbf{e}_b$). This leads to a more efficient "pruning" of the (context $\times$ action) space, reducing the cumulative sum $\sum s_{t,\tau_t}$.
The combination of these effects results in $\mathcal{R}_T(SC\text{-}LinUCB) \le \mathcal{R}_T(LinUCB\text{-}NS)$, with significant improvement when semantic features offer substantial parsimony or fit advantages.
\end{proof}
\section{ SC-LinUCBExperiments} 
\label{sec:linucb_experiments}

To empirically evaluate the impact of semantic information in contextual bandit settings, we employ two variants of the shared LinUCB algorithm \citep{AbbasiYadkori2011Improved}. Both agents aim to learn a single shared parameter vector $\bm{\theta}^* \in \mathbb{R}^d$ to predict expected rewards $\mathbb{E}[R_t | \vect{x}_{t,j}] \approx \vect{x}_{t,j}^T \bm{\theta}^*$. Their core distinction lies in the construction of the feature vector $\vect{x}_{t,j}$ for a given query (context) $\vect{q}_t$ and tool (action) $\tau_j$.

We compare SC-LinUCB and LinUCB-OneHot using their respective semantic and non-semantic feature constructions detailed in Section \ref{subsec:sclinucb_features}. For this experiment with $K=6$ tools, $d_{sem}=6$ and $d_{non-sem}=9$. Results are averaged over $N_{runs}=15$ seeds.

We conduct a series of experiments to empirically validate our theoretical findings and demonstrate the practical benefits of using semantic action features. We first focus on an intra-episode setting with a fixed action set, then evaluate adaptation in a continual learning scenario with dynamic action sets. Experiments are run on Colab (free tier CPU).

\subsection{Experiment 1: Intra-Episode Efficiency in a Multi-Context Environment}
\label{sec:exp1_main_paper_final}

\paragraph{Objective.}
This experiment validates our theoretical claim that SC-LinUCB achieves lower regret than LinUCB-OneHot by leveraging semantic action features in a multi-context setting with a fixed action set ($K=6$).

\paragraph{Environment Setup.}
The environment features $N_Q=3$ distinct query types (contexts) that cycle periodically over $T=10000$ timesteps. Each of the $K=6$ tools $\tau_j$ is associated with a 2D toy semantic embedding $\vect{\phi}_j$, derived from one of three underlying archetypes plus noise. Each query type is designed to align semantically with one specific tool archetype. Stochastic rewards $R_t \in \{0,1\}$ are determined by the semantic alignment between the current query $\vect{q}_t$ and the chosen tool's embedding $\vect{\phi}_j$. Full details are in Appendix \ref{app:exp1_details}.

\paragraph{Results.}
Figure \ref{fig:exp1_regret_main_final} presents the average cumulative regret (log scale). SC-LinUCB (orange line) shows substantially superior performance, maintaining an exceptionally low cumulative regret (around $10^0$) over $10000$ timesteps, indicating rapid convergence to a nearly optimal policy across contexts. LinUCB-OneHot (blue line), while exhibiting sublinear regret (indicating learning), incurs orders-of-magnitude higher regret (exceeding $10^3$). This stark difference underscores SC-LinUCB's ability to generalize semantic patterns across different (context, tool) pairings, leading to vastly improved sample efficiency compared to the baseline, which learns tool utilities more independently. Both algorithms used $\alpha=0.3$. For ablations over the value of $\alpha$ we refer to figure \ref{fig:exp1_full_appendix}. 

\begin{figure}[htbp]
    \centering
    \includegraphics[width=\linewidth]{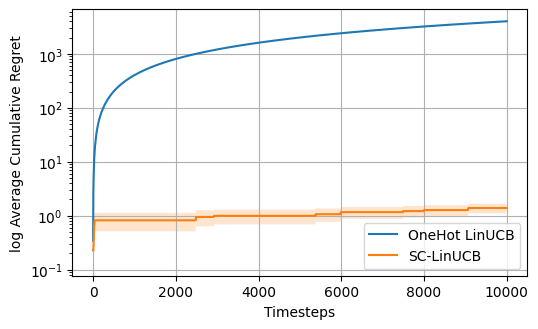} 
    \caption{Average Cumulative Regret (log scale) for SC-LinUCB and LinUCB-OneHot in the multi-context (switching) with fixed toolset experiemnt. Time steps $T=10000$, averaged over $15$ runs and  $\alpha = 0.3$).}
    \label{fig:exp1_regret_main_final}
    \end{figure}

\subsection{Experiment 2: Continual Adaptation to Dynamic Toolsets}
\label{sec:exp2_main_paper_revisedplot}

\paragraph{Objective.}
This experiment evaluates the agents' ability to adapt to a dynamically changing tool set over four distinct phases ($T_{phase}=2500$ steps each, for a total of $T=10000$ steps), involving tool addition, removal, and the introduction of novel semantic types alongside new relevant queries. The setup tests the robustness and generalization capabilities crucial for lifelong learning. 
The environment cycles through three base query types ($\vect{q}_A, \vect{q}_B, \vect{q}_C$) for the first three phases, with a fourth query type ($\vect{q}_D$) introduced in Phase 4. 
Full phase details, including specific tool archetype assignments and query cycling, are in Appendix \ref{app:exp2_full_plots_and_alphas}. LinUCB-OneHot re-initializes it's model matrices ($A,b$) when $K$ changes due to its feature space dependency on $K$. SC-LinUCB's model matrices and $d_{sem}$ remain fixed. Both agents use an exploration parameter $\alpha=0.5$ for this illustrative plot (sensitivity to $\alpha$ is explored in Appendix \ref{app:exp2_full_plots_and_alphas}). Results are averaged over $N_{runs}=15$ independent seeds.

\paragraph{Results.}
Figure \ref{fig:exp2_regret_main_newplot} (see figure \ref{fig:exp2_multialpha_appendix_plot} for corresponding reward plots) illustrates the average cumulative regret on a log scale.
\begin{figure}[htbp]
    \centering
    \includegraphics[width=\linewidth]{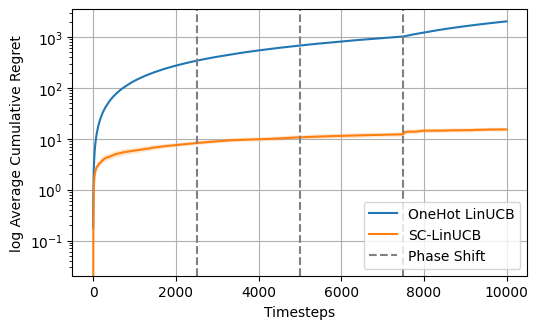} 
    \caption{Average Cumulative Regret (log scale) for SC-LinUCB and LinUCB-OneHot in the continual adaptation experiment. Each phase is $2500$ steps, changes indicated by dashed lines.} 
    \label{fig:exp2_regret_main_newplot}
\end{figure}
The performance of \textbf{SC-LinUCB (Semantic, orange line)} is remarkably robust. Its cumulative regret remains very low, consistently around $10^1$ (approximately 10-20 units), across all four phases and $10000$ time steps. Crucially, at the phase transitions (dashed vertical lines at $t=2500, 5000, 7500$), its regret curve shows almost no perturbation. This demonstrates SC-LinUCB's ability to gracefully handle tool removal, leverage its existing semantic knowledge to quickly incorporate new tools with familiar semantic embeddings (Phase 3), and effectively learn about novel semantic types when new queries make them relevant (Phase 4), all without catastrophic forgetting or costly re-learning phases.

In stark contrast, \textbf{LinUCB-OneHot (Non-Semantic, blue line)} exhibits significantly higher regret and poor adaptation. Its regret climbs steeply, exceeding $10^3$ by the end of the experiment. At each phase transition where the number of tools $K$ changes, its regret curve shows a pronounced upward jump or a sharply increased slope. This is a direct consequence of its model matrices ($A,b$) being re-initialized due to the change in its feature space dimensionality ($d_{non-sem} = d_q + K + 1$), forcing it to largely relearn the value of tools from scratch for the new configuration.

These results strongly underscore the high cost of adaptation for a non-semantic agent in dynamic environments. SC-LinUCB's fixed-dimensional semantic feature space, combined with its capacity for semantic generalization, provides robust, efficient, and truly continual learning in the face of a changing action landscape.

\section{SC in LLM Tool Orchestrators}
\label{sec:icl_experiments}

Using and training LLM to orchestrate across $O$ many tools can be done in a broad variety of methods. As previously mentioned it can be e.g. a classic policy mapping the query to an action (id or name) or a policy taking in the query alongside the semantic context. 
Crucially there is a variety of training regimes. A popular branch of methods used LLM fine-tuning techniques (full rank or low rank) using supervised fine-tuning \cite{prabhakar2025apigenmtagenticpipelinemultiturn} with RL reasoning \cite{feng2025retoolreinforcementlearningstrategic, zhang2025nemotronresearchtooln1exploringtoolusinglanguage} and algorithms like PPO or GRPO. All of these provide semantic context in their implementations. An alternative is to follow the recipe in \cite{DulacArnold2015DeepRL} and train a hierarchical policy that predicts in the first step for a given query a text description of the action it wants to take (or an embedding of the action) and performs in the second stage nearest-neighbour search/ softmax over k-nearest neighbours to select the respective action. A third method is to rely on the in-context learning abilities. 

We rigorously evaluate how SC impacts LLM in-context learning efficacy for sequential tool selection. We frame this as a multi-armed bandit (MAB) problem: an LLM agent learns to select optimal tools based on query context and interaction history presented via its prompt. Our investigation spans static and dynamic environments, assessing fundamental learning and adaptation.




\subsection{Experimental Design}
\label{ssec:experimental_design_main}

Our experimental design focuses on varying the semantic richness of action representations provided to the LLM. We consider four conditions:

     \textbf{Index Only (IO)}: Actions are presented as abstract, non-informative indices (e.g., ``Action 1'', ``Action 2''). This baseline tests the LLM's ability to learn solely from correlations in the interaction history.
     
         \textbf{Name Only (NO)}: Actions are presented by their names (e.g., ``Data Analyzer'', ``QuickTranslate''). This provides a concise signal, yet it is quite fragile. 
         
     \textbf{Name + Description (ND)}: Actions are presented with their names and detailed functional descriptions, offering the richest semantic context.
     
         \textbf{Description Only (DO)}: Actions are presented as abstract non-informative indices together with detailed functional descriptions.

The LLM for all experiments is \texttt{Gemini 2.0 Flash}. Each experiment was conducted for multiple independent trials (5 for static environments, 7 for dynamic environments). The full prompt structure, LLM parameters (temperature \(0.5\), max output tokens \(500-1500\)), and detailed configurations of arms and queries are provided in appendix \ref{app:experimental_details_llm_icl}. We report the average return over trials, where the expectation is taken over the stochasticity of rewards and LLM responses in figure \ref{fig:summary_reward_plots}. Average cumulative regrets are presented in figure \ref{fig:summary_regret_plots}.

We designed four distinct experimental scenarios:
\textbf{Exp 1} \textit{(fQfA)}: fixed query and fixed tools probes baseline in-context learning of best arm selection; \textbf{Exp 2} \textit{(mQfA)}: varied queries with fixed tools test contextual generalisation; \textbf{Exp 3} \textit{(fQmA)}: fixed query with evolving tools measures adaptation; \textbf{Exp 4} \textit{(mQmA)}: both queries and tools shift, stressing full non-stationary robustness.


\subsection{Results and Analysis}
\label{ssec:results_analysis_main}

\begin{figure*}[htbp]
    \centering
    \includegraphics[width=\textwidth]{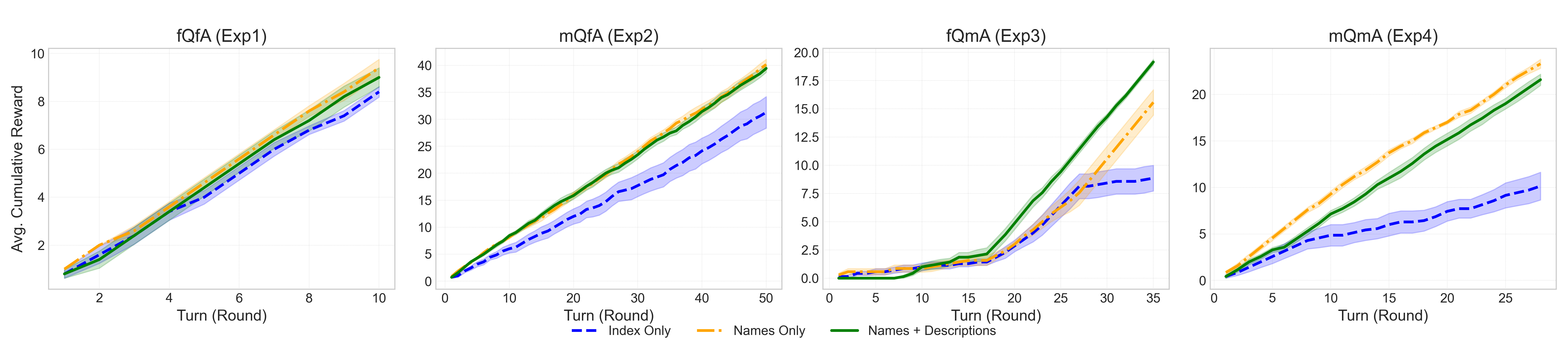} 
    \caption{Semantic Context yields higher average return across Experiments 1-4. Subplot titles indicate: f=Fixed, m=Moving, Q=Queries, A=Actions. Shaded regions represent \(\pm 1\) standard error of the mean (SEM) across trials. Higher values indicate better performance. Note the varying x and y-axis scales.}
    \label{fig:summary_reward_plots}
\end{figure*}

The experimental results, depicted by the average cumulative reward curves in \ref{fig:summary_reward_plots}, reveal a nuanced and significant impact of semantic context on the LLM's in-context learning and adaptation for tool selection. For the corresponding regret plot we refer the reader to figure \ref{fig:summary_regret_plots} in the appendix. With the small action gap and the poor performance of the index only, the cumulative reward plot tells the semantic baselines better apart. 

\textbf{Static Environments (fQfA - Exp1; mQfA - Exp2):}
In environments with fixed action spaces (Exp1 and Exp2 panels in \ref{fig:summary_reward_plots}), providing richer semantic context generally leads to higher cumulative rewards. ND (green solid line) and NO (orange dash-dot line) both outperform IO (blue dashed line). In Exp1 (fQfA), ND and NO perform very similarly, both achieving near-optimal reward accumulation, indicating that even names are sufficient for the single, repeated query. In Exp2 (mQfA), which involves multiple queries, ND maintains a slight edge over NO, suggesting that descriptions help differentiate tools more effectively as contextual complexity increases. IO consistently lags, demonstrating the LLM's difficulty in accumulating rewards without semantic cues to guide its choices.

\textbf{Dynamic Environments (fQmA - Exp3; mQmA - Exp4):}
The introduction of non-stationarity through changing action spaces and/or queries highlights more complex interactions.
In Experiment 3 (fQmA: fixed query, moving actions), the reward plot (\ref{fig:summary_reward_plots}, Exp3 panel) shows that the ND condition adapts most effectively to the introduction of a superior tool (``E3\_SuperCalc'') around turn 17 (phase details in \ref{app:exp3_details}). Its reward accumulation rate increases sharply after this point, surpassing NO. The NO condition also shows adaptation and reward growth but appears to either identify or commit to the superior tool with a delay or less consistency. The IO condition is slow in picking up the dynamic reward signal. 

Experiment 4 (mQmA: moving queries and actions) presents the most striking results (\ref{fig:summary_reward_plots}, Exp4 panel). In this highly dynamic scenario, the NO achieves the highest cumulative reward, notably outperforming ND. This intriguing outcome suggests that when both tasks and tools are frequently changing, the conciseness of tool names might offer an advantage in terms of agile decision-making or reduced risk of misinterpretation compared to longer descriptions. The ND condition still performs well and significantly better than IO, but the simpler NO representation appears more robust or efficient for the LLM in this specific rapidly evolving context. The IO condition again accumulates the least reward, underscoring its inadequacy in complex dynamic settings.

\textbf{Synthesis of Findings:}
The results consistently demonstrate that the absence of semantic context (IO) severely limits the LLM's ability to effectively learn and accumulate rewards in these tool-selection bandit tasks. Providing semantic information is crucial. Rich descriptions (ND) are generally powerful, particularly for rapid optimal tool identification in static settings and for adapting to clear improvements (like a new superior tool for a known task) in dynamic environments. However, the superior performance of Names Only in the most complex, fully dynamic scenario (Exp4) is a key finding. It suggests a potential trade-off: while detailed descriptions offer depth, concise and sufficiently indicative names might facilitate greater agility or reduce the processing/interpretation burden on the LLM when faced with rapid and multifaceted environmental changes. This implies that the optimal level of semantic detail for action representation may not be monolithic but rather depend on the stability and complexity of the operational environment. It is also crucial to reiterate the finding of \cite{lin2024hammerrobustfunctioncallingondevice} that the naming of tools alters much more across developers than descriptions, making this approach more fragile.

\textbf{Analysis of the Reasoning traces}

Reasoning traces from Gemini 2.0 Flash (detailed examples in App.~\ref{app:subsec_icl_reasoning_traces}) reveal \textit{how} LLMs leverage semantic context. 
\begin{reasontrace}
"Reasoning: The query mentions \"sales figures\" and \"growth pattern\", which indicates numerical data and the need to find trends. The Data Analyzer tool is specifically designed for processing numerical data arrays to find trends. The other tools, Text Formatter and Image Resizer, are not relevant to this query. Therefore, the Data Analyzer is the most likely to yield a reward.\n\n Chosen Action: Data Analyzer"
\end{reasontrace}
Two main decision mechanisms are apparent: the reliance on previous experience (ICL exploration \citep{krishnamurthy2024largelanguagemodelsexplore}), and semantic matching. Particularly with ND and NO, the LLM often engages in a two-step semantic matching process: 1) analyse the user query to infer the abstract capability required; 2) match this inferred need against the semantic information of available tools, selecting the best aligner. This resembles the two-step action selection \citep{DulacArnold2015DeepRL} where the policy maps first to a desired description (proto action) and subsequently selects the most appropriate match. 
For instance, for a sales growth query, the LLM with ND or NO typically identifies a "Data Analyser" by matching functionality. The richness of ND can lead to more nuanced initial alignments (Exp3), while the conciseness of NO might offer faster, if less precise, matching in dynamic scenarios (Exp4), potentially reducing cognitive load. This relies however on the concise tool naming ability of the tool creator. \cite{lin2024hammerrobustfunctioncallingondevice} raise that tool and argument naming is more user-sensitive than the function description, making the latter more robust. Crucially, NO and even more ND can enable LLM to prioritize semantic fit over immediate past negative rewards for the best tool. In contrast, IO relies solely on the ICL ability of LLM. The observed two-step reasoning provides a qualitative explanation for SC's quantitative benefits, suggesting that LLMs internalize descriptions for structured decision-making beyond simple index-based pattern matching.

\subsection{Semantic Context for Scaling Action Space }
\label{subsec:sc_scaling_action_space}

\begin{figure*}[htbp]
    \centering
    \includegraphics[width=1.0\textwidth]{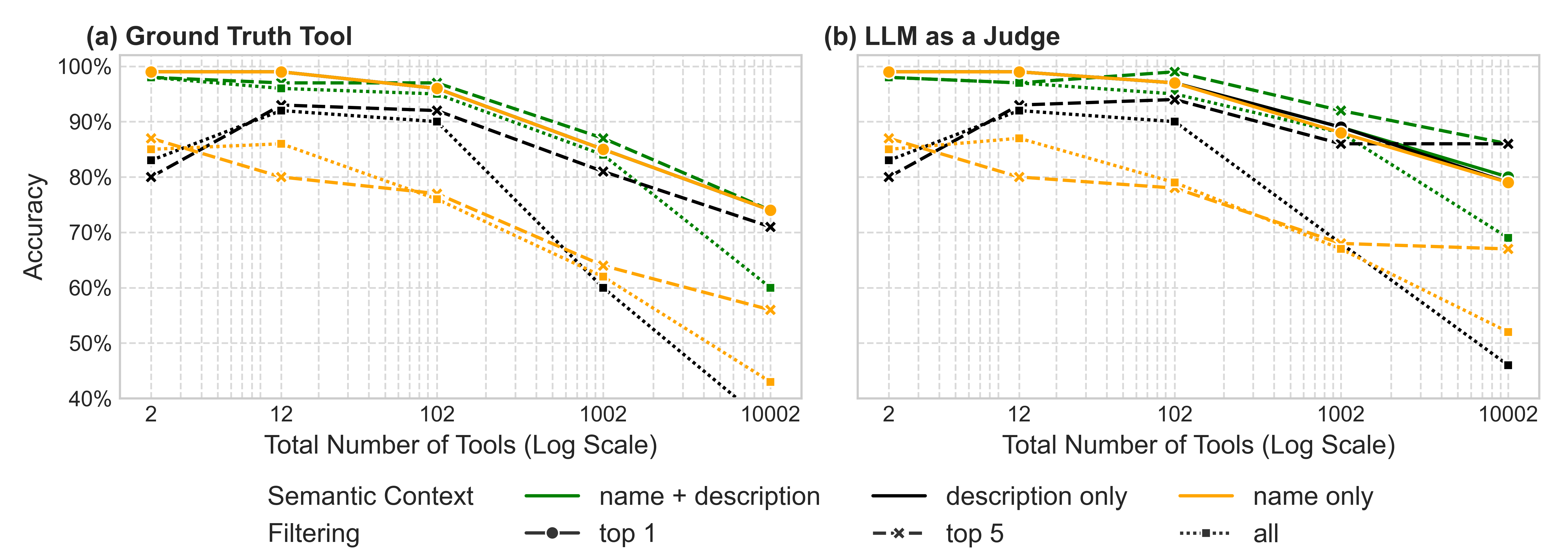} 
    \caption{Semantic Context is essential for scalable tool selection with \text{top 5} filtering followed by ND yields the strongest performance for large tool sets. Accuracy is plotted against the total number of tools (log scale). The left plot shows accuracy of identifying the ground truth tool, whereas the right plot uses an LLM as judge to evaluate the tool correctness.} 
    \label{fig:0_shot_generalisation_different_tool_numbers}
\end{figure*}

In the previous subsection reasoning traces showed a two step of action description and action selection pattern. In all this experiments all tools and descriptions were part of the policy LLM context. To be practical, an orchestrator must scale to large amounts of tools. As the context of LLM runs naturally at some point out, we propose a "filter-then-reason-then-act" (FiReAct) pipeline. Pseudocode of FiReAct is provided in alg \ref{alg:fireact} and can be thought of as Tool-RAG version of RAG \cite{lewis2021retrievalaugmentedgenerationknowledgeintensivenlp}.
\begin{algorithm}[H]
\caption{The FiReAct Pipeline}
\label{alg:fireact}
\begin{algorithmic}[1]
    \REQUIRE Embedding model $\phi$, LLM policy $\pi$, query $q_t$, tool set $\mathcal{A}_t$, num candidates $k$
    \STATE {\bf Filter:} Retrieve candidate subset $\mathcal{A}_{cand} \subseteq \mathcal{A}_t$ of size $k$ via semantic search using $\phi(q_t)$ and $\{\phi(a^i)|a^i \in \mathcal{A}_t\}$.
    \STATE {\bf Reason \& Act:} Select final action $a_{selected} \in \mathcal{A}_{cand}$ using the LLM policy $\pi(q_t, C_S(\mathcal{A}_{cand}))$.
    \STATE Return: {$a_{selected}$}
\end{algorithmic}
\end{algorithm}
There exist a variety of methods to filter for the candidate action set $\mathcal{A}_{cand}$. One could for example simply ask an LLM to do it. We instantiate the FiReAct pipeline using a \texttt{text-embedding-004} retriever and a \texttt{gemini-2.0-flash} LLM policy. Firstly query and tools are embedded and the top $k$ tools selected. These are feed (in the respective descriptive format (IO, ND,NO,DO) together with the query to the LLM policy. Based on this, the tool is selected.
FireAct can be deployed at both test and train time. We demonstrate its usage at test time in a  0-shot pipeline on a challenging benchmark constructed from the XLAM dataset \cite{zhang2024xlamfamilylargeaction}, evaluating 100 queries against a corpus of over 10,000 tools. 
Figure~\ref{fig:0_shot_generalisation_different_tool_numbers} plots tool selection accuracy for three strategies: pure semantic retrieval (`top 1`), LLM-filtered reasoning (`top 5`), and exhaustive unfiltered reasoning (`all`). The results are unequivocal: without SC, performance is catastrophic. The IO condition yields 0-shot just random pulls, thus $(1/O)$ success rate. 

Given SC's necessity, its quality is paramount. Rich ND context (green lines) consistently provides the highest accuracy across all methods, offering a distinct advantage over the weaker `name only` and `description only` signals. This shows that while any semantic signal is beneficial, more detailed information provides critical disambiguation power, especially as the number of distractor tools increases. Note however the superiour/ competitive performance of NO with N+D for up to 100 distractor tools. This demonstrates that more detailed semantic information provides critical disambiguation power in complex environments. However less SC (NO) is sometimes simpler, we hypothesize due to the smaller context window. 

The most crucial finding, however, reveals how to best leverage SC at scale. While pure retrieval (`top 1`) is powerful, its top-1 precision degrades as the tool space grows; with 10,000 distractors, the accuracy for `name + description` context falls to $~75\%(80\% \text{ with LLM Judge})$. The retriever's recall within the top 5 remains high, however, creating a vital opportunity for a reasoning step. By having the LLM re-rank these `top 5` candidates, we restore accuracy to nearly 90\%. This ~15\% accuracy gain 
validates the FiReAct pipeline as a robust, scalable strategy, where SC is the essential fuel for both initial filtering and final, high-fidelity reasoning.

\section{Future Work and Conclusion}
\label{sec:future_work_conclusion}

This paper establishes that explicit Semantic Context (SC), derived from action descriptions, is a powerful asset for efficient tool orchestration. We demonstrated this principle across two distinct paradigms: for linear contextual bandits, we proved that our \mbox{SC-LinUCB} framework enables more efficient learning and robust adaptation to dynamic action spaces compared to non-semantic baselines. We then showed that these principles translate directly to Large Language Models performing in-context learning. Our experiments confirmed that richer semantic context generally enhances tool selection and adaptability, and we demonstrated that our FiReAct pipeline leverages SC to make this approach scalable to thousands of tools.

Our work has limitations that open clear avenues for future investigation. Theoretically, key steps include developing sharp regret bounds for formally non-stationary toolsets ($\mathcal{A}_t$) and analyzing algorithmic robustness to noisy or imperfect semantic features. While our LLM experiments are indicative, they are specific to the chosen model and prompting strategies; direct theoretical guarantees for in-context tool learning remain a major open challenge. Empirically, extending our experiments to the fine-tuning of LLMs and developing end-to-end trainable retrieval-reasoning pipelines are promising directions.

In conclusion, by formalizing and demonstrating the ``semantic advantage,'' our work shows that leveraging the inherent meaning of actions is a more effective strategy than treating them as opaque indices. The consistent benefits of SC observed across different learning frameworks—from linear models to large transformers—suggest that providing structured, semantic descriptions of actions is a valuable and generalizable design principle. This approach provides a principled path toward developing agents that are more sample-efficient, adaptive, and scalable when interacting with complex and evolving toolsets.

\bibliographystyle{apalike}
\bibliography{references}

\appendix


\onecolumn
\section{Background}
\subsection{Notation at a Glance}
\label{ref:app_notation_overview}
\begin{table}[h]
  \centering
  \caption{Notation at a glance}
  \small
  \begin{tabular}{ll}
    \toprule
    Symbol & Meaning \\\midrule
    $\mathcal{A}_t $ & action set available at round $t$ of cardinality $O_t$ \\
    $\phi_t(a)$     & semantic feature vector of action $a$ \\
    $d_{sem}$         & similarity metric on $\mathcal{X}$ \\
    $\theta_\star$  & unknown linear reward vector \\
    $V_t$           & design matrix at round $t$ \\
    \bottomrule
  \end{tabular}
\end{table}

\subsection{Semantic Context MDP}

\label{app:subsec_sc_mdp}
\begin{definition}[Semantic Context MDP, SC-MDP]
\label{def:sc_mdp_v5}
An \textbf{SC-MDP} describes sequential decision-making with a \textit{fixed} toolset $\mathcal{A}_{avail}$ and its corresponding fixed Semantic Action Context $C_S(\mathcal{A}_{avail})$. It is an MDP $(\mathcal{S}, \mathcal{A}, P, R, \gamma)$ where:
     The state $s_t \in \mathcal{S}$ is typically $(h_t, q_t)$, representing history and current query.
     The action space $\mathcal{A}$ consists of choices $(a_j, \text{args}(a_j))$ where $a_j \in \mathcal{A}_{avail}$.
     The policy $\pi(a_t | s_t)$ implicitly utilizes the fixed $C_S(\mathcal{A}_{avail})$ (which defines this specific MDP environment) to select $a_t$.
     Transitions $P(s_{t+1} | s_t, a_t)$ and rewards $R(s_t, a_t)$ are standard. Tool execution yields an output $o_t$, forming part of $h_{t+1}$.
\end{definition}

\begin{definition}[Lifelong Semantic Context MDP, LSC-MDP]
\label{def:lsc_mdp_v5}
An \textbf{LSC-MDP} models scenarios with a \textit{dynamically changing} tool set $\mathcal{A}_t$. It is an MDP $(\mathcal{S}_{LSC}, \mathcal{A}_{LSC}, P_{LSC}, R_{LSC}, \gamma)$, where the state $s_t \in \mathcal{S}_{LSC}$ is $(h_t, q_t, C_S(\mathcal{A}_t))$, explicitly includes the time-dependent SC $C_S(\mathcal{A}_t)$ that changes as the tool set $\mathcal{A}_t$ evolves.
     The action space $\mathcal{A}_{LSC}(s_t)$ comprises choices $(a_j, \text{args}(a_j))$ where $a_j \in \mathcal{A}_t$.
     The policy is $\pi(a_t | s_t)$.
     Transition dynamics $P_{LSC}(s_{t+1} | s_t, a_t)$ determine the next query $q_{t+1}$ and, crucially, the next available toolset $\mathcal{A}_{t+1}$ (and thus $C_S(\mathcal{A}_{t+1})$).
\end{definition}

\section{Appendix Semantic Context LINUCB}
\label{app:proofs_and_algos_sc_linucb}

\subsection{Formal Assumptions}
\label{app:formal_assumptions_sclinucb}
For the linear bandit setting we have the following standard assumptions. 
\begin{assumption}[Contextual Linear Bandit Setting (Restated)]
Over $T$ timesteps, $t \in \{1, \dots, T\}$:
\begin{enumerate}
    \item A context $s_t$ is observed, from which a $d_q$-dimensional query embedding $\vect{q}_t = q(s_t)$ is derived.
    \item The agent selects an action (tool) $a_t$ from a fixed set of $K$ tools $\mathcal{A} = \{a_1, \dots, a_K\}$.
    \item Each tool $a_j \in \mathcal{A}$ has a $d_{desc}$-dimensional semantic description embedding $\bm{\phi}_j = \phi(D_{a_j})$.
    \item For each context-tool pair $(q_t, a_j)$, a $d$-dimensional feature vector $\vect{x}_{t,j} = x(\vect{q}_t, \bm{\phi}_j)$ is constructed. We assume $\|\vect{x}_{t,j}\|_2 \le L_x$.
    \item The expected reward is linear in these features: $\mathbb{E}[R_t(\vect{x}_{t,j}) | \vect{x}_{t,j}] = \vect{x}_{t,j}^T \bm{\theta}^*$ for an unknown true parameter vector $\bm{\theta}^* \in \mathbb{R}^d$. We assume $\|\bm{\theta}^*\|_2 \le S_\theta$.
    \item Observed rewards are $R_t(\vect{x}_{t,j}) = \vect{x}_{t,j}^T \bm{\theta}^* + \eta_{t,j}$, where $\eta_{t,j}$ is conditionally $\sigma$-subGaussian noise: $\mathbb{E}[\eta_{t,j} | \vect{x}_{t,j}] = 0$ and $\mathbb{E}[e^{\lambda \eta_{t,j}} | \vect{x}_{t,j}] \le e^{\lambda^2 \sigma^2 / 2}$ for all $\lambda \in \mathbb{R}$.
\end{enumerate}
\end{assumption}

\begin{algorithm}[tb]
   \caption{Bubble Sort}
   \label{alg:example}
\begin{algorithmic}
   \STATE {\bfseries Input:} data $x_i$, size $m$
   \REPEAT
   \STATE Initialize $noChange = true$.
   \FOR{$i=1$ {\bfseries to} $m-1$}
   \IF{$x_i > x_{i+1}$}
   \STATE Swap $x_i$ and $x_{i+1}$
   \STATE $noChange = false$
   \ENDIF
   \ENDFOR
   \UNTIL{$noChange$ is $true$}
\end{algorithmic}
\end{algorithm}

\subsection{SC-LinUCB Algorithm Detail}
\label{app:algo_details}
\begin{algorithm}[H]
\caption{SC-LinUCB (Shared Model) - Appendix Version}
\label{alg:sc_linucb_appendix}
\begin{algorithmic}[1]
\REQUIRE Exploration parameter $\alpha > 0$, regularization $\lambda_{reg} > 0$.
\STATE Initialize $\matr{A} = \lambda_{reg} \matr{I}_d$, $\vect{b} = \mathbf{0}_d$.
\FOR{$t = 1, \dots, T$:}
    \STATE Observe query $\vect{q}_t$.
    \STATE For each tool $a_j \in \mathcal{A}$ (with semantic embedding $\bm{\phi}_j$), construct feature vector $\vect{x}_{t,j} = x(\vect{q}_t, \bm{\phi}_j)$.
    \STATE Compute $\matr{A}^{-1}$. 
    \STATE Compute $\hat{\bm{\theta}}_t = \matr{A}^{-1} \vect{b}$.
    \STATE For each tool $a_j \in \mathcal{A}$:
    \STATE \quad $s_{t,j} \leftarrow \sqrt{\vect{x}_{t,j}^T \matr{A}^{-1} \vect{x}_{t,j}}$
    \STATE \quad $p_{t,j} \leftarrow \vect{x}_{t,j}^T \hat{\bm{\theta}}_t + \alpha s_{t,j}$
    \STATE Choose $a_t = \argmax_{j \in \{1, \dots, K\}} p_{t,j}$ (break ties randomly).
    \STATE Let $\vect{x}_t^{chosen} = \vect{x}_{t,a_t}$.
    \STATE Play tool $a_t$, observe reward $R_t(\vect{x}_t^{chosen})$.
    \STATE $\matr{A} \leftarrow \matr{A} + \vect{x}_t^{chosen} (\vect{x}_t^{chosen})^T$.
    \STATE $\vect{b} \leftarrow \vect{b} + R_t(\vect{x}_t^{chosen}) \vect{x}_t^{chosen}$.
\ENDFOR
\end{algorithmic}
\end{algorithm}

\subsection{Standard Lemmas and Proof for Generic LinUCB Regret}
\label{app:std_lemmas_proofs}

\begin{theorem}
[Confidence Set for $\bm{\theta}^*$, Theorem 2 from \citet{AbbasiYadkori2011Improved}]
\label{lemma:appendix_concentration}
Under Assumption \ref{ass:semantically_structured_rewards}, let $\delta \in (0,1)$ and $\lambda_{reg} > 0$. Define
$$ \alpha_t'(\delta) \coloneqq \sigma \sqrt{2 \log(1/\delta) + d \log\left(1+\frac{t L_x^2}{\lambda_{reg}d}\right)} + \sqrt{\lambda_{reg}}S_\theta $$
(This form of $\alpha$ is closer to the direct statement in Abbasi-Yadkori et al., Theorem 2, which uses $\log(\det(\matr{A}_t)/\det(\lambda_{reg}\matr{I})) \le d \log(1+tL_x^2/(\lambda_{reg}d))$).
Then, with probability at least $1-\delta$, for all $t \ge 1$, $\bm{\theta}^*$ lies in the set $C_t = \{ \bm{\theta} \in \mathbb{R}^d : \|\hat{\bm{\theta}}_t - \bm{\theta}\|_{\matr{A}_t} \le \alpha_t'(\delta) \}$.
This implies that for any $\vect{x} \in \mathbb{R}^d$ with $\|\vect{x}\|_2 \le L_x$,
$ |\vect{x}^T (\hat{\bm{\theta}}_t - \bm{\theta}^*)| \le \alpha_t'(\delta) \sqrt{\vect{x}^T \matr{A}_t^{-1} \vect{x}} $.
For the main paper, we use a slightly simplified $\alpha \ge \alpha_T'(\delta)$ for clarity, which might incorporate a $\log K$ term for uniform convergence over arms at each step if not absorbed into $\delta$.
\end{theorem}
\begin{proof}
See proof of theorem 2 from \citet{AbbasiYadkori2011Improved} for full derivation.
\end{proof}

\begin{lemma}[Elliptical Potential Lemma, Lemma 11 from \citet{AbbasiYadkori2011Improved}]
\label{lemma:appendix_elliptical_potential}
Let $\vect{x}_1, \dots, \vect{x}_T \in \mathbb{R}^d$ be a sequence of feature vectors such that $\|\vect{x}_t\|_2 \le L_x$. Let $\matr{A}_t = \lambda_{reg}\matr{I}_d + \sum_{j=1}^{t-1} \vect{x}_j \vect{x}_j^T$. Then,
$$ \sum_{t=1}^T \min(1, \vect{x}_t^T \matr{A}_t^{-1} \vect{x}_t) \le 2d \log\left(1 + \frac{T L_x^2}{\lambda_{reg}d}\right) $$
\end{lemma}
\begin{proof}
See proof of Lemma 11 from \citet{AbbasiYadkori2011Improved}.
\end{proof}

\subsection{Elleptical potential lemma}
We restate and proof the elleptical potential lemma:

\begin{lemma}[Elliptical Potential Lemma, Lemma 11 from \citet{AbbasiYadkori2011Improved}]
Let $\vect{x}_1, \dots, \vect{x}_T \in \mathbb{R}^d$ be a sequence of feature vectors such that $\|\vect{x}_t\|_2 \le L_x$. Let $\matr{A}_t = \lambda_{reg}\matr{I}_d + \sum_{j=1}^{t-1} \vect{x}_j \vect{x}_j^T$. Then,
$$ \sum_{t=1}^T \min(1, \vect{x}_t^T \matr{A}_t^{-1} \vect{x}_t) \le 2d \log\left(1 + \frac{T L_x^2}{\lambda_{reg}d}\right) $$
If $\lambda_{reg} \ge L_x^2$, then $\vect{x}_t^T \matr{A}_t^{-1} \vect{x}_t \le \vect{x}_t^T (\lambda_{reg}\matr{I}_d)^{-1} \vect{x}_t = \frac{\|\vect{x}_t\|^2}{\lambda_{reg}} \le \frac{L_x^2}{\lambda_{reg}} \le 1$, so the $\min(1,\cdot)$ can be removed. For a general $\lambda_{reg}$, the bound still holds with the $\min$.
\end{lemma}
\begin{proof}
See Lemma 11 and Appendix A.3 in \citet{AbbasiYadkori2011Improved}. 
\end{proof}

\subsection{Detailed Argument for Theorem \ref{thm:sc_linucb_adv_semantic_generalization} (Advantage of SC-LinUCB)}
\label{app:proof_adv_semantic_generalization}

Theorem \ref{thm:sc_linucb_adv_semantic_generalization} posits that SC-LinUCB achieves lower regret than LinUCB-NS by enabling more efficient exploration and generalization through its semantic features. We elaborate on the two main mechanisms:

\textbf{1. More Parsimonious Effective Model (Relating to $d$):}
The regret bound for LinUCB scales roughly with $d$, the feature dimensionality.
For SC-LinUCB, features $\vect{x}^{(sem)}_{t,j} = [\vect{q}_t; \bm{\phi}_j; \text{sim}(\vect{q}_t, \bm{\phi}_j); 1]$ have dimension $d_{sem} = d_q + d_{desc} + 1 + 1$.
For LinUCB-NS with one-hot tool encodings, $\vect{x}^{(non-sem)}_{t,j} = [\vect{q}_t; \mathbf{e}_j; 1]$ has dimension $d_{non-sem} = d_q + K + 1$.

Assumption \ref{app:formal_assumptions_sclinucb} implies that the true reward function $f^*(\vect{q}_t, \bm{\phi}_j)$ depends on shared semantic properties encoded in $\bm{\phi}_j$ and their interaction with $\vect{q}_t$. If the diversity of $K$ tools can be meaningfully captured by $d_{desc}$-dimensional semantic embeddings such that $d_{desc} \ll K$ (e.g., tools fall into fewer semantic archetypes than $K$, or their reward-relevant variations are low-dimensional), then $d_{sem}$ can be substantially smaller than $d_{non-sem}$.
SC-LinUCB learns a single parameter vector $\hat{\bm{\theta}}_{sem} \in \mathbb{R}^{d_{sem}}$. This vector effectively models the utility of semantic *attributes* (dimensions of $\vect{q}_t$, dimensions of $\bm{\phi}_j$, and their similarity) and how they combine to predict reward. This model is shared across all $K$ tools.
LinUCB-NS, on the other hand, needs to learn parameters in $\hat{\bm{\theta}}_{non-sem} \in \mathbb{R}^{d_{non-sem}}$ where $K$ of these dimensions (from $\mathbf{e}_j$) are dedicated to capturing the unique identity and behavior of each tool. If there is underlying semantic redundancy across tools that LinUCB-NS cannot exploit, it is effectively learning a higher-dimensional model than necessary.
Thus, if $d_{sem} < d_{non-sem}$ and both feature sets achieve a comparable quality of linear approximation (i.e., $\sigma_{eff,sem} \approx \sigma_{eff,non-sem}$), the $d$ factor in the regret bound directly favors SC-LinUCB. This represents a reduction in the complexity of the parameter space to be learned.

\textbf{2. Faster Reduction of Uncertainty for Semantically Similar Options (Relating to $\sum s_{t,a_t}$):}
The instantaneous regret $r_t$ is bounded by $2\alpha s_{t,a_t} = 2\alpha \sqrt{\vect{x}_{t,a_t}^T \matr{A}_t^{-1} \vect{x}_{t,a_t}}$. The cumulative regret depends on the sum of these exploration terms.
Consider the update to the covariance matrix $\matr{A}_{t+1} = \matr{A}_t + \vect{x}_t \vect{x}_t^T$. The inverse $\matr{A}_{t+1}^{-1}$ shrinks based on the direction of $\vect{x}_t$. The exploration term $s_{t',j}^2 = \vect{x}_{t',j}^T \matr{A}_{t+1}^{-1} \vect{x}_{t',j}$ for any arm $j$ at a future step $t'$ will decrease more significantly if $\vect{x}_{t',j}$ has a substantial component along the direction of $\vect{x}_t$ (the chosen arm's features at time $t$).

For SC-LinUCB, if tool $a_a$ is chosen at time $t$ (with features $\vect{x}^{(sem)}_{t,a}$), the update to $\matr{A}_{sem}$ reflects increased certainty along the semantic dimensions present in $\vect{x}^{(sem)}_{t,a}$. Now, consider another tool $a_b$. If $a_b$ is semantically similar to $a_a$ with respect to context $\vect{q}_t$ (or a similar context $\vect{q}_{t'}$), then their feature vectors $\vect{x}^{(sem)}_{t,a}$ and $\vect{x}^{(sem)}_{t',b}$ will share many active semantic components (e.g., similar $\bm{\phi}$ components, similar interaction terms). Consequently, the exploration term $s_{t',b}^{(sem)}$ for tool $a_b$ will also be reduced due to the information gained from pulling $a_a$. The agent effectively learns about a "semantic neighborhood" of tools with each pull.

For LinUCB-NS, the feature vectors $\vect{x}^{(non-sem)}_{t,a} = [\vect{q}_t; \mathbf{e}_a; 1]$ and $\vect{x}^{(non-sem)}_{t,b} = [\vect{q}_t; \mathbf{e}_b; 1]$ (for $a \ne b$) have orthogonal tool-identity components $\mathbf{e}_a$ and $\mathbf{e}_b$. An update from pulling $a_a$ (involving $\mathbf{e}_a$) primarily reduces uncertainty associated with $\mathbf{e}_a$ and its interaction with $\vect{q}_t$. It has minimal effect on reducing the uncertainty associated with the distinct orthogonal direction $\mathbf{e}_b$. Thus, LinUCB-NS learns little about $a_b$'s specific utility from pulling $a_a$, even if $a_a$ and $a_b$ are semantically very similar.

This implies that SC-LinUCB can "cross off" or gain confidence about larger regions of the (context $\times$ semantic tool property) space with each observation. As a result, the sum of exploration terms $\sum_{t=1}^T s_{t,a_t}$ is expected to be smaller for SC-LinUCB compared to LinUCB-NS over $T$ steps, as it requires fewer "distinctly exploratory" pulls to identify good actions across the spectrum of contexts and tools. While the Elliptical Potential Lemma (Lemma \ref{lemma:appendix_elliptical_potential}) bounds $\sum s_{t,a_t}^2$ by $O(d \log T)$ for both, the actual sequence of $s_{t,a_t}$ values chosen by SC-LinUCB can be smaller on average due to this generalization, leading to a tighter sum for $\sum s_{t,a_t}$ when applying Cauchy-Schwarz.

Combining a potentially smaller $d_{sem}$ with a more efficient exploration dynamic (leading to a smaller effective sum of exploration bonuses), SC-LinUCB achieves lower cumulative regret.

\subsection{SC-LinUCB in the continual setting}
\label{ss:sc_linucb_continual}
Beyond efficiency with a fixed set of tools, SC-LinUCB's semantic feature design offers significant advantages in continual learning scenarios where the set of available tools $\mathcal{A}_t$ (and thus its size $K_t$) changes over time. This is a critical capability for agents in evolving environments.

Consider a setting with phases, where within each phase $p$, the toolset $\mathcal{A}^{(p)}$ is fixed, but it can change between phases (e.g., $\mathcal{A}^{(p+1)} = (\mathcal{A}^{(p)} \setminus \mathcal{A}_{removed}) \cup \mathcal{A}_{added}$).

\begin{theorem}[Low-Cost Adaptation of SC-LinUCB to Dynamic Toolsets]
\label{thm:continual_adaptation_sc_linucb}
Let SC-LinUCB use semantic features $\vect{x}^{(sem)}$ of fixed dimension $d_{sem}$ and LinUCB-NS use one-hot features $\vect{x}^{(non-sem)}$ of dimension $d_{non-sem}(K_t) = d_q + K_t + 1$.
When the set of available tools changes from $\mathcal{A}^{(p)}$ (size $K^{(p)}$) to $\mathcal{A}^{(p+1)}$ (size $K^{(p+1)}$):

\begin{enumerate}
    \item \textbf{SC-LinUCB (Semantic):}
        \begin{itemize}
            \item Its feature dimension $d_{sem}$ remains constant.
            \item Its learned parameter vector $\hat{\bm{\theta}}_{sem}^{(p)}$ (from phase $p$) and covariance matrix $\matr{A}_{sem}^{(p)}$ remain valid and are directly carried over to phase $p+1$.
            \item For any newly added tool $a_{new} \in \mathcal{A}_{added}$ with semantic embedding $\bm{\phi}_{new}$, SC-LinUCB can immediately compute its feature vector $\vect{x}^{(sem)}_{q, new}$ and estimate its utility using the existing $\hat{\bm{\theta}}_{sem}^{(p)}$, yielding an informed initial UCB score.
            \item The "cost of adaptation" is primarily the exploration required for new semantic aspects introduced by $\mathcal{A}_{added}$ that were not sufficiently covered by $\hat{\bm{\theta}}_{sem}^{(p)}$. If new tools are semantically similar to previously seen optimal tools, adaptation is very fast.
        \end{itemize}
    \item \textbf{LinUCB-NS (Non-Semantic Baseline):}
        \begin{itemize}
            \item If $K^{(p+1)} \ne K^{(p)}$, its feature dimension $d_{non-sem}(K_t)$ changes. This necessitates a change in its parameter vector $\hat{\bm{\theta}}_{non-sem}$ and matrices $\matr{A}_{non-sem}, \vect{b}_{non-sem}$.
            \item Common strategies for LinUCB-NS include:
                \textit{(a) Full Re-initialization:} $\matr{A}_{non-sem}$ and $\vect{b}_{non-sem}$ are reset. The agent effectively relearns from scratch for the new toolset $\mathcal{A}^{(p+1)}$, incurring regret similar to starting a new bandit problem of size $K^{(p+1)}$.
                \textit{(b) Heuristic Adaptation:} Attempting to adapt $\matr{A}_{non-sem}, \vect{b}_{non-sem}$ (e.g., adding/removing rows/columns) is complex and typically still treats new tool IDs as completely unknown entities requiring extensive exploration.
            \item For any newly added tool $a_{new}$, LinUCB-NS has no prior information derived from other tools about its utility, as its one-hot encoding is orthogonal to others.
            \item The "cost of adaptation" involves significant relearning for the entire (or substantial parts of) the new toolset.
        \end{itemize}
\end{enumerate}
Consequently, over a sequence of phases with changing toolsets, SC-LinUCB is expected to achieve substantially lower cumulative regret than LinUCB-NS due to its fixed-dimensional semantic representation, knowledge transfer via $\hat{\bm{\theta}}_{sem}$, and ability to gracefully incorporate or ignore tools based on their semantic features without model restructuring.
\end{theorem}

\begin{proof}[Proof Sketch for Theorem \ref{thm:continual_adaptation_sc_linucb}]
This theorem's argument builds on the properties of the agents and the implications of Theorem \ref{thm:sc_linucb_adv_semantic_generalization} applied piecewise.

\textbf{For SC-LinUCB:}
The feature space $\mathbb{R}^{d_{sem}}$ and the parameter vector $\bm{\theta}_{sem}^*$ are defined over semantic properties, not tool identities or the count $K_t$. Thus, the learned model $(\hat{\bm{\theta}}_{sem}, \matr{A}_{sem})$ retains its validity and utility when the set of available tools $\mathcal{A}_t$ changes.
\begin{itemize}
    \item \textit{Tool Addition:} When $a_{new}$ (with $\bm{\phi}_{new}$) is added, SC-LinUCB calculates $\vect{x}^{(sem)}_{q,new}$ and its UCB score using the current $\hat{\bm{\theta}}_{sem}$ and $\matr{A}_{sem}$. If $\bm{\phi}_{new}$ aligns semantically with query features for which $\hat{\bm{\theta}}_{sem}$ has learned high weights, $a_{new}$ will be explored efficiently. The exploration cost is for resolving uncertainty about this specific $\vect{x}^{(sem)}_{q,new}$ within the existing learned model structure. No part of the model needs to be "resized" or "reset."
    \item \textit{Tool Removal:} If $a_{removed}$ is removed, SC-LinUCB simply no longer considers it for selection. Its learned $\hat{\bm{\theta}}_{sem}$ and $\matr{A}_{sem}$ (which contain information from past pulls of $a_{removed}$) remain valid for evaluating the remaining tools.
\end{itemize}
The regret within any phase $p$ where $\mathcal{A}^{(p)}$ is fixed is governed by Theorem \ref{thm:sc_linucb_adv_semantic_generalization} with $d=d_{sem}$. The transitions between phases incur minimal structural cost.

\textbf{For LinUCB-NS (OneHot):}
The feature space $\mathbb{R}^{d_{non-sem}(K_t)}$ explicitly depends on the current number of tools $K_t$ via the one-hot encodings $\mathbf{e}_j \in \mathbb{R}^{K_t}$.
\begin{itemize}
    \item \textit{Tool Addition (K increases):} $d_{non-sem}$ increases. The matrices $\matr{A}_{non-sem}$ and $\vect{b}_{non-sem}$ must be expanded. The new dimensions corresponding to the new tool ID have no prior history. Effectively, the agent must learn about this new tool's interaction with all query types from scratch. If the agent fully resets $\matr{A}_{non-sem}, \vect{b}_{non-sem}$ (as done in our Experiment 2 for a clear baseline), it starts a new learning problem with regret $\tilde{O}(d_{non-sem}(K_{new})\sqrt{T_{phase}})$. Even with more sophisticated matrix adaptation, the components of $\bm{\theta}_{non-sem}^*$ relevant to the new tool are unknown.
    \item \textit{Tool Removal (K decreases):} $d_{non-sem}$ decreases. The agent might discard rows/columns from $\matr{A}_{non-sem}, \vect{b}_{non-sem}$. This is less detrimental than addition if no reset occurs, but the overall problem structure for its features has changed.
\end{itemize}
The key issue is that LinUCB-NS's learned knowledge is tied to specific tool indices. If these indices change, or new ones appear, extensive relearning is often needed for those affected dimensions. The strategy of re-initializing $\matr{A}, \vect{b}$ upon change in $K$ (as implemented for LinUCB-OneHot in our Experiment 2) represents a clear case where it incurs a full bandit learning cost for the new configuration.

\textbf{Comparing Adaptation Costs:}
The "cost" can be seen as the additional regret incurred during a phase transition compared to an oracle that was already adapted.
For SC-LinUCB, this cost is low because $\hat{\bm{\theta}}_{sem}$ provides immediate, semantically-informed estimates for new tools, and its structure is stable.
For LinUCB-NS (with resets on $K$ change), this cost is high, equivalent to the initial regret of a new bandit problem.
Thus, over multiple phases of toolset changes, the cumulative regret of SC-LinUCB will be substantially lower due to these significantly reduced adaptation costs at phase boundaries, on top of its potential intra-phase efficiency from Theorem \ref{thm:sc_linucb_adv_semantic_generalization}.
\end{proof}

\subsection{Experiment 1: Detailed Setup and Full Results for Intra-Episode Efficiency}
\label{app:exp1_details}

This section provides further details for Experiment 1, which evaluates the intra-episode efficiency of SC-LinUCB with semantic features against LinUCB-OneHot with non-semantic features in a multi-context toy environment.

\paragraph{Environment Design.}
The environment is a contextual bandit task designed to highlight the benefits of semantic generalization.
\begin{itemize}
    \item \textbf{Timesteps ($T$):} Each experimental run consists of $T=10000$ timesteps.
    \item \textbf{Tools ($K$):} There are $K=6$ tools available throughout each run.
    \item \textbf{Tool Semantic Archetypes and Embeddings ($\bm{\phi}_j$):} Tools are designed around $N_{arch}=3$ underlying semantic archetypes. Each tool $a_j$ is assigned one of these archetypes. Its $d_{tool\_sem}=2$ dimensional toy semantic embedding $\vect{\phi}_j$ is generated by taking the corresponding archetype vector and adding Gaussian noise with zero mean and standard deviation $\sigma_{emb\_noise}=0.05$. This noise is re-generated for each of the $N_{runs}$ independent experimental trials to ensure robustness of results to minor variations in embeddings.
        The archetype vectors are:
        \begin{itemize}
            \item Archetype 1 ($\vect{\phi}_{arch1}$): $[0.9, 0.1]^T$ (2 tools assigned this archetype)
            \item Archetype 2 ($\vect{\phi}_{arch2}$): $[0.1, 0.9]^T$ (2 tools assigned this archetype)
            \item Archetype 3 ($\vect{\phi}_{arch3}$): $[-0.7, -0.7]^T$ (2 tools assigned this archetype, replacing the previous 1 'type3' and 1 'noise' for more symmetry)
        \end{itemize}
    \item \textbf{Queries/Contexts ($\vect{q}_t$):} There are $N_Q=3$ distinct query types, each represented by a $d_q=2$ dimensional toy embedding. These queries cycle periodically every $N_Q$ timesteps (i.e., $q_A, q_B, q_C, q_A, q_B, q_C, \dots$).
        The query embeddings are:
        \begin{itemize}
            \item Query A ($\vect{q}_A$): $[1.0, 0.2]^T$, designed to align best with Tool Archetype 1.
            \item Query B ($\vect{q}_B$): $[0.2, 1.0]^T$, designed to align best with Tool Archetype 2.
            \item Query C ($\vect{q}_C$): $[-0.8, -0.8]^T$, designed to align best with Tool Archetype 3.
        \end{itemize}
    \item \textbf{Reward Function ($R_t$):} The reward $R_t \in \{0,1\}$ is stochastic, drawn from a Bernoulli distribution. The success probability $P(\text{success} | \vect{q}_t, \vect{\phi}_j)$ is determined by the semantic alignment between the current query $\vect{q}_t$ and the chosen tool's embedding $\vect{\phi}_j$. Specifically:
    $$ P(\text{success}) = \text{clip}(P_{base} + C_{sim} \cdot (\vect{q}_t^T \vect{\phi}_j) + B_{align}, P_{min}, P_{max} ) $$
    where $P_{base}=0.45$ is a base success rate, $C_{sim}=0.40$ scales the dot product similarity, and $B_{align}=0.25$ is a bonus awarded if the chosen tool's true archetype matches the current query's preferred archetype. Probabilities are clipped to $[P_{min}=0.05, P_{max}=0.95]$. This structure ensures that tools whose semantic embeddings align well with the current query, especially those of the preferred archetype, have a higher expected reward.
\end{itemize}

\paragraph{Agent Configurations.}
Both SC-LinUCB and LinUCB-OneHot are instances of the stanard LinUCB algorithm differing only in their feature construction:
\begin{itemize}
    \item \textbf{SC-LinUCB (Semantic):} Uses $d_{sem} = d_q + d_{tool\_sem} + 1 (\text{similarity}) + 1 (\text{bias}) = 2+2+1+1=6$ dimensional features: $\vect{x}^{(sem)}_{t,j} = [\vect{q}_t; \vect{\phi}_j; \vect{q}_t^T \vect{\phi}_j; 1]$.
    \item \textbf{LinUCB-OneHot (Non-Semantic Baseline):} Uses $d_{non-sem} = d_q + K + 1 (\text{bias}) = 2+6+1=9$ dimensional features: $\vect{x}^{(non-sem)}_{t,j} = [\vect{q}_t; \mathbf{e}_j; 1]$, where $\mathbf{e}_j$ is the one-hot encoding for tool $a_j$.
\end{itemize}
Both agents use $\lambda_{reg}=1.0$. We evaluate exploration parameters $\alpha \in \{0.3, 0.5, 1.0\}$.

\paragraph{Evaluation Metrics.}
Results are averaged over $N_{runs}=15$ independent Monte Carlo runs. We report:
\begin{enumerate}
    \item \textbf{Average Cumulative Reward:} $\frac{1}{N_{runs}} \sum_{run=1}^{N_{runs}} \sum_{t=1}^{T} R_t^{(run)}$.
    \item \textbf{Average Cumulative Regret:} $\frac{1}{N_{runs}} \sum_{run=1}^{N_{runs}} \sum_{t=1}^{T} (\mathbb{E}[R|\vect{q}_t, a_t^*] - \mathbb{E}[R|\vect{q}_t, a_t^{(run)}])$. Here, $\mathbb{E}[R|\vect{q}_t, a]$ is the true expected reward (success probability) of tool $a$ for query $\vect{q}_t$, and $a_t^*$ is the tool with the maximum expected reward for $\vect{q}_t$. This uses expected instantaneous regret for smoother non-decreasing cumulative regret curves.
\end{enumerate}

\paragraph{Full Experimental Results.}
Figure \ref{fig:exp1_full_appendix} shows both the average cumulative reward and average cumulative regret on logarithmic y-axes for all tested $\alpha$ values.

\begin{figure}[htbp]
    \centering
    \includegraphics[width=\textwidth]{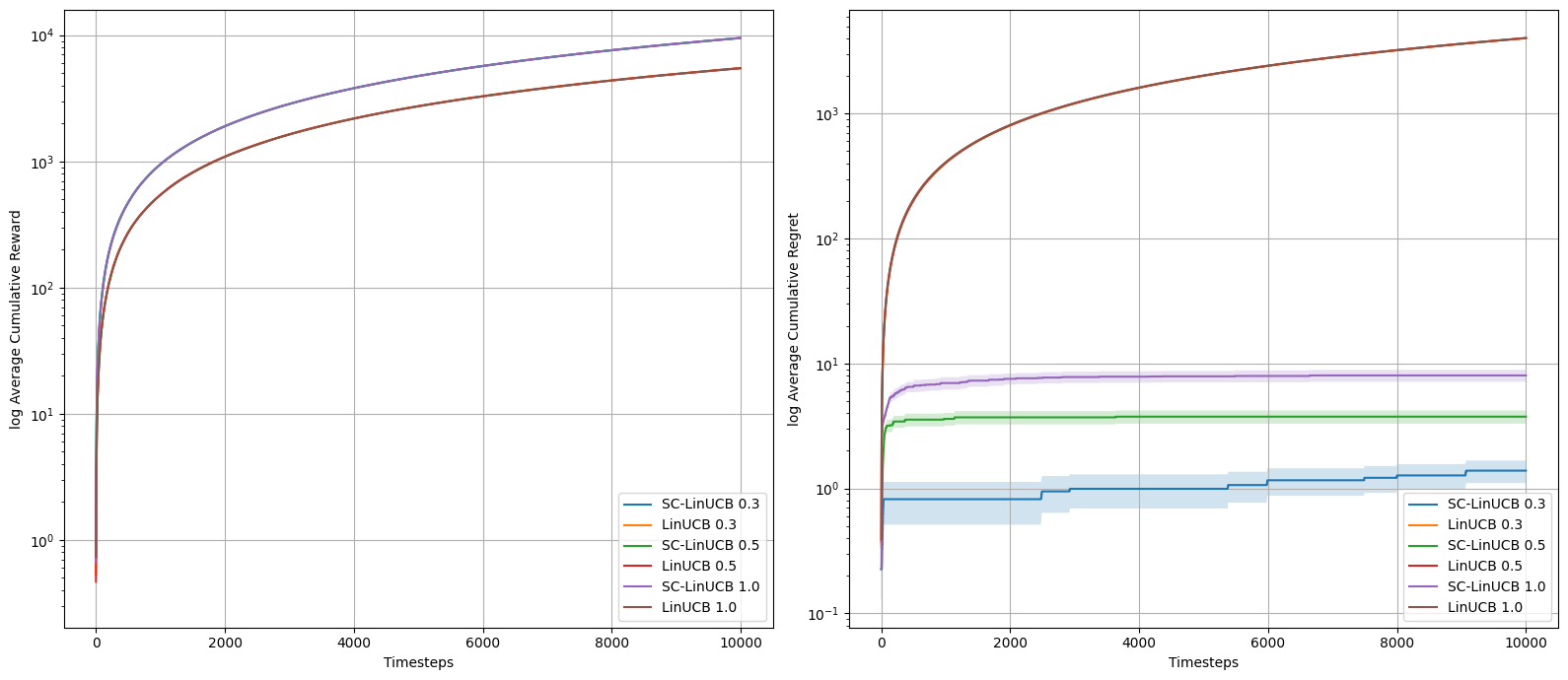} 
    \caption{Full results for Experiment 1: SC-LinUCB (Semantic) vs. LinUCB-OneHot (Non-Semantic) in the multi-context toy environment ($T=10000$, $15$ runs). Left: Average Cumulative Reward (log scale). Right: Average Cumulative Regret (log scale). Different line styles/colors within agent types correspond to $\alpha \in \{0.3, 0.5, 1.0\}$.}
    \label{fig:exp1_full_appendix}
\end{figure}

The results clearly indicate the superiority of SC-LinUCB. In the reward plot (left), SC-LinUCB variants (particularly with $\alpha=1.0$, purple dashed line) accumulate substantially more reward over time compared to LinUCB-OneHot variants. The log scale emphasizes the sustained higher rate of reward collection.

The regret plot (right) offers the most striking comparison. SC-LinUCB agents maintain an extremely low cumulative regret (primarily between $10^0$ and $10^1$), indicating rapid convergence to near-optimal policies for the cycling contexts. The SC-LinUCB (Semantic) $\alpha=0.3$ (blue solid line) shows the lowest regret overall. In stark contrast, all LinUCB-OneHot variants incur regret that is orders of magnitude higher, reaching $10^3$. While their regret curves are sublinear (indicating learning), their inefficiency compared to SC-LinUCB is evident. The LinUCB-OneHot agent with $\alpha=1.0$ (brown solid line) performs best among the non-semantic baselines but is still vastly outperformed.

These empirical findings strongly corroborate our theoretical analysis (Theorem \ref{thm:sc_linucb_adv_semantic_generalization}). The ability of SC-LinUCB to generalize across tools and contexts using a compact semantic feature space ($d_{sem}=6$) leads to substantially more efficient learning than LinUCB-OneHot, which must learn more independently for each tool ID within its higher-dimensional feature space ($d_{non-sem}=9$). The semantic features provide a powerful inductive bias that aligns with the problem structure, reducing the effective complexity faced by the learning algorithm.

\subsection{Experiment 2: Detailed Results for Continual Adaptation with Varying Exploration}
\label{app:exp2_full_plots_and_alphas}

This section provides the full results for Experiment 2, which evaluates the continual adaptation capabilities of SC-LinUCB (Semantic) and LinUCB-OneHot (Non-Semantic) in an environment with dynamically changing toolsets. We present a sensitivity analysis with respect to the exploration parameter $\alpha \in \{0.3, 0.5, 1.0\}$.

\paragraph{Experimental Setup Recap.}
The environment consists of four distinct phases, each lasting $T_{phase}=2500$ timesteps (total $T=10000$). The set of available tools ($K$) and active query types ($N_Q$) evolve across these phases, involving tool addition (of both semantically familiar and novel types), tool removal, and the introduction of new query types corresponding to novel tools.
\begin{itemize}
    \item \textbf{Phase 1 ($K=4, N_Q=3$):} Initial tools: $\{a_{A1}, a_{A2} (\text{type1}); a_{B1}, a_{B2} (\text{type2})\}$. Queries: $\vect{q}_A, \vect{q}_B, \vect{q}_C$.
    \item \textbf{Phase 2 ($K=3, N_Q=3$):} Tool $a_{A2}$ (type1) removed. (Starts at $t=2500$)
    \item \textbf{Phase 3 ($K=4, N_Q=3$):} New tool $a_{A3}$ (type1, semantically similar to $a_{A1}$) added. (Starts at $t=5000$)
    \item \textbf{Phase 4 ($K=5, N_Q=4$):} New tool $a_{D1}$ (novel semantic type4) added; query $\vect{q}_D$ (aligning with type4) becomes active. (Starts at $t=7500$)
\end{itemize}
LinUCB-OneHot re-initializes its model matrices ($A,b$) when $K$ changes. SC-LinUCB's core model matrices and semantic feature dimension ($d_{sem}=6$) remain fixed. Toy embeddings and the reward function are as described in Appendix \ref{app:exp1_details} (or a dedicated Exp2 setup section if it differs significantly). All results are averaged over $N_{runs}=15$ independent seeds.

\paragraph{Results with Varying Alphas.}
Figure \ref{fig:exp2_multialpha_appendix_plot} displays the average cumulative reward (left, log scale) and average cumulative regret (right, log scale) for both SC-LinUCB and LinUCB-OneHot across the three tested values of $\alpha$.

\begin{figure}[htbp]
    \centering
    \includegraphics[width=\textwidth]{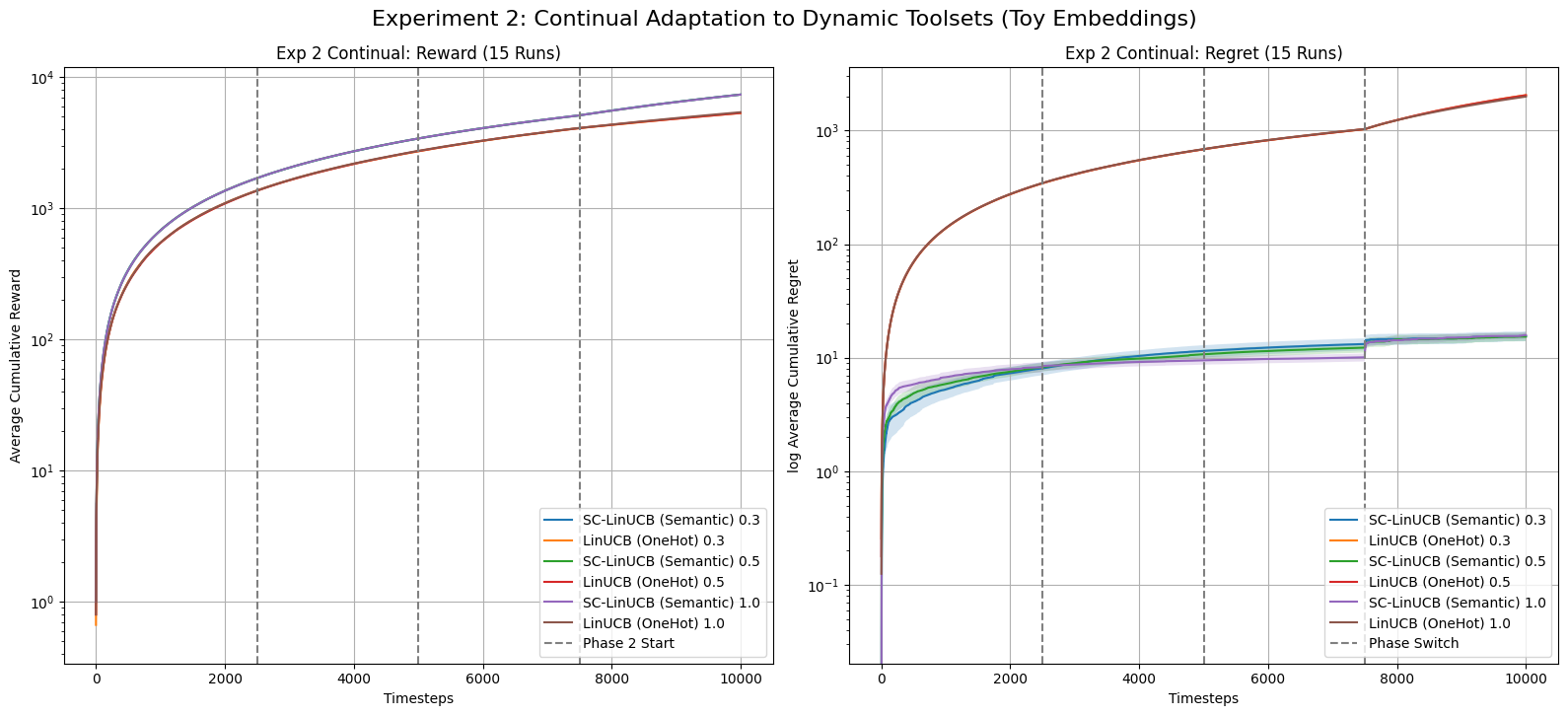} 
    \caption{Experiment 2 (Continual Adaptation): Performance of SC-LinUCB (Semantic) and LinUCB-OneHot (Non-Semantic) with varying exploration parameters $\alpha \in \{0.3, 0.5, 1.0\}$. Results over $4 \times 2500$ timesteps, averaged over $15$ runs. Vertical dashed lines indicate phase shifts. Left: Average Cumulative Reward (log scale). Right: Average Cumulative Regret (log scale).}
    \label{fig:exp2_multialpha_appendix_plot}
\end{figure}

\textbf{Cumulative Reward Analysis (Figure \ref{fig:exp2_multialpha_appendix_plot}, Left):}
SC-LinUCB variants consistently achieve higher cumulative rewards than LinUCB-OneHot variants across all tested $\alpha$ values. For SC-LinUCB, $\alpha=1.0$ (purple dashed line) yields the highest overall reward, suggesting that with strong semantic features, a reasonably high level of exploration can be beneficial for maximizing long-term reward, even in a changing environment. For LinUCB-OneHot, $\alpha=1.0$ (brown dashed line) is also its best configuration, but it still lags significantly behind all SC-LinUCB variants. The SC-LinUCB curves maintain a steadier rate of reward accumulation across phase transitions, whereas the LinUCB-OneHot curves show more pronounced slowdowns or changes in slope, indicative of their relearning periods.

\textbf{Cumulative Regret Analysis (Figure \ref{fig:exp2_multialpha_appendix_plot}, Right):}
The regret plot starkly illustrates the advantages of SC-LinUCB.
\begin{itemize}
    \item \textbf{SC-LinUCB (Semantic):} All variants (blue $\alpha=0.3$, green $\alpha=0.5$, purple $\alpha=1.0$) maintain exceptionally low cumulative regret, generally staying within the $10^0$ to $10^1$ range over $10000$ steps. The phase transitions cause only minor, temporary increases in regret, from which they recover quickly. SC-LinUCB with $\alpha=0.3$ and $\alpha=0.5$ show particularly stable and low regret. The $\alpha=1.0$ variant, while achieving high rewards, exhibits slightly higher regret and notably wider variance (shaded area), especially around phase shifts, likely due to more extensive exploration when the environment changes. This indicates that while higher exploration can find good policies, it might come at the cost of some initial suboptimality if the semantic signal is already strong.
    \item \textbf{LinUCB-OneHot (Non-Semantic):} All variants incur substantially higher regret, ending up in the $10^2$ to $10^3$ range. Crucially, at each phase transition where $K$ changes (vertical dashed lines), there is a distinct upward turn or steepening of the regret slope. This clearly visualizes the significant cost of adaptation incurred by LinUCB-OneHot as it re-initializes its model and relearns the utility of tools largely from scratch. Increased exploration (e.g., $\alpha=1.0$, brown line) helps LinUCB-OneHot achieve lower regret compared to its lower $\alpha$ counterparts, but it remains orders of magnitude worse than any SC-LinUCB variant.
\end{itemize}

\paragraph{Conclusion from Alpha Sensitivity.}
SC-LinUCB demonstrates robust superiority over LinUCB-OneHot across the tested range of exploration parameters in this continual learning setting. Its ability to leverage fixed-dimensional semantic features allows for graceful adaptation to dynamic toolsets with minimal regret cost. While LinUCB-OneHot does benefit from increased exploration, its fundamental inability to generalize semantically across tools and its need to restructure its feature space when the number of tools changes impose a significant and persistent learning burden. For the main paper, we typically present results for a representative $\alpha$ (e.g., $\alpha=0.5$) that showcases good performance for SC-LinUCB, as seen in Figure \ref{fig:exp2_regret_main_newplot}. This detailed ablation confirms the general trends.
\section{Appendix ICL Experiments}
\label{app:icl_exp} 

\subsection{Experimental Setup Details}
\label{app:experimental_details_llm_icl}

This section provides comprehensive details of the configurations used for all experiments discussed in the main paper, ensuring reproducibility.

\subsubsection{LLM Parameters and Prompt Structure}
\label{app:llm_prompt_details}
The Large Language Model (LLM) utilized across all four experiments was Gemini 2.0 Flash, accessed via the \texttt{models/gemini-2.0-flash} API endpoint.
Key generation parameters were consistently set as follows:
\begin{itemize}
    \item Temperature: \(0.5\)
    \item Maximum Output Tokens: \(500\) for Experiments 1 \& 2; \(1500\) for Experiments 3 \& 4 (to accommodate potentially longer reasoning with dynamic changes).
\end{itemize}
No specialized safety settings beyond API defaults were applied.

The fundamental prompt structure provided to the LLM comprised a system message defining the task and action presentation, followed by the interaction history and the current query.

\textbf{System Prompt Template:}
\begin{promptbox}
You are an intelligent assistant playing a multi-armed bandit game.
Your goal is to maximize your total reward over many turns. 
The available actions (tools) or types of queries may change over time.
In each turn, you are presented with a user query and a list of currently 
available actions. Each action, when chosen for a query it is suited for, 
has a specific hidden probability of yielding a reward of 1, and 0 otherwise.
If an action is not suited for the query, or no suitable action is available, 
it will likely yield a reward of 0.
You must choose one action if suitable options exist. 
If no actions are available or suitable, state that.

Available actions: [
  {Formatted list of actions based on experimental condition}
]
\end{promptbox}
The placeholder \texttt{\{Formatted list of actions...\}} was populated according to the active experimental condition (Index Only, Names Only, Description Only or Names + Descriptions) for the currently available tools in that phase/turn.

\textbf{User Message Template per Turn:}
\begin{promptbox}
Interaction History (most recent 20 turns shown for LLM context):
{Interaction history string, e.g., 
Turn 1: Query: "Full Query Text 1", Your Choice: ActionName1, Outcome: Reward 0
...
Turn K: Query: "Full Query Text K", Your Choice: ActionNameK, Outcome: Reward R_K
}

Current User Query (Global Turn {current_global_turn}): "{Current Query Text}"

Think step-by-step about which action is best for the current query. 
Consider the query, CURRENTLY available action descriptions, and past experiences. 
After your reasoning, state your final choice clearly. 
For example: "Reasoning: [...reasons...]. Chosen Action: ActionName Or Index". 
If no action is suitable or available, you can state 'Chosen Action: None'.
Which action do you choose?
\end{promptbox}

The interaction history provided in the prompt to the LLM contained the full text of the past 20 queries, chosen actions, and their rewards. The experimental framework maintained the complete history for logging and analysis. Each experiment was run for a set number of independent trials: 5 trials for Experiments 1 and 2 (static), and 7 trials for Experiments 3 and 4 (dynamic).

\subsubsection{Experiment 1 (fQfA) Configuration Details}
\label{app:exp1_details}
\begin{itemize}
    \item \textbf{Description}: Single query repeated for \(T=10\) turns, fixed action space.
    \item \textbf{Query (\texttt{q\_analyze})}: ``I have a list of sales figures for the last quarter, can you help me understand the growth pattern?'' (Optimal Arm: \texttt{tool\_A})
    \item \textbf{Arm Configurations}:
    \begin{itemize}
        \item \texttt{tool\_A} (Data Analyzer): ``Processes numerical data arrays to find trends.'' (\(p^{\text{true}}=0.9, p^{\text{subopt}}=0.55\))
        \item \texttt{tool\_B} (Text Formatter): ``Cleans and formats long text strings.'' (Designed with \(p^{\text{true}}=0.9\), used with \(p^{\text{subopt}}=0.5\) when chosen for \texttt{q\_analyze})
        \item \texttt{tool\_C} (Image Resizer): ``Changes the dimensions of image files.'' (Designed with \(p^{\text{true}}=0.8\), used with \(p^{\text{subopt}}=0.6\) when chosen for \texttt{q\_analyze})
    \end{itemize}
\end{itemize}

\subsubsection{Experiment 2 (mQfA) Configuration Details}
\label{app:exp2_details}
\begin{itemize}
    \item \textbf{Description}: Queries randomly drawn from a fixed set for \(T=50\) turns, fixed action space.
    \item \textbf{Arm Configurations}:
        \item \texttt{tool\_translate} (QuickTranslate): ``Translates short text snippets between common languages.'' (\(p^{\text{true}}=0.85, p^{\text{subopt}}=0.5\))
        \item \texttt{tool\_summarize} (BriefSummary): ``Creates a one-sentence summary of a paragraph.'' (\(p^{\text{true}}=0.75, p^{\text{subopt}}=0.5\))
        \item \texttt{tool\_calendar} (EventScheduler): ``Adds events to a user's primary calendar.'' (\(p^{\text{true}}=0.9, p^{\text{subopt}}=0.55\))
        \item \texttt{tool\_filesearch} (DocFinder): ``Searches for local documents by keyword.'' (\(p^{\text{true}}=0.7, p^{\text{subopt}}=0.6\))
    \item \textbf{Query Configurations (Randomly Sampled from this set)}:
        \item \texttt{q\_trans\_hello}: ``How do you say 'hello' in Spanish?'' (Optimal: \texttt{tool\_translate})
        \item \texttt{q\_sum\_paragraph}: ``Give me the gist of this: 'The quick brown fox jumps over the lazy dog every day.''' (Optimal: \texttt{tool\_summarize})
        \item \texttt{q\_sched\_meeting}: ``Schedule a meeting with Jane for tomorrow at 2 PM.'' (Optimal: \texttt{tool\_calendar})
        \item \texttt{q\_find\_report}: ``Find the Q3 sales report document on my drive.'' (Optimal: \texttt{tool\_filesearch})
        \item \texttt{q\_trans\_bye}: ``What is 'goodbye' in French?'' (Optimal: \texttt{tool\_translate})
        \item \texttt{q\_sum\_news}: ``Briefly, what's this news about: 'Local team wins championship after a dramatic final.'?'' (Optimal: \texttt{tool\_summarize})
\end{itemize}

\subsubsection{Experiment 3 (fQmA) Configuration Details}
\label{app:exp3_details}
\begin{itemize}
    \item \textbf{Description}: Single query repeated for \(T=35\) turns (total across phases), action space changes in phases.
    \item \textbf{Query (\texttt{Q\_ComplexMath})}: ``Solve the integral of x \^{} 2 * sin(x) from 0 to pi, and also find the square root of 1764.'' (Designated Optimal Arm (when available): \texttt{E3\_SuperCalc})
    \item \textbf{Master Arm Configurations}:
        \item \texttt{E3\_Calculator} (Basic Calculator): ``Performs simple arithmetic (+, -, *, /).'' (\(p^{\text{true}}=0.7, p^{\text{subopt}}=0.1\))
        \item \texttt{E3\_SciCalculator} (Scientific Calculator): ``Advanced math functions: exponents, logs, trig.'' (\(p^{\text{true}}=0.9, p^{\text{subopt}}=0.15\))
        \item \texttt{E3\_UnitConverter} (Unit Converter): ``Converts units (e.g., kg to lbs, meters to feet).'' (\(p^{\text{true}}=0.8, p^{\text{subopt}}=0.05\))
        \item \texttt{E3\_Plotter} (Data Plotter): ``Generates simple plots from data.'' (\(p^{\text{true}}=0.6, p^{\text{subopt}}=0.1\))
        \item \texttt{E3\_SuperCalc} (SuperMath Solver): ``Handles complex algebra, calculus, and symbolic math. The ultimate math tool.'' (\(p^{\text{true}}=0.95, p^{\text{subopt}}=0.2\))
    \item \textbf{Phase Details (Total 35 Turns)}:
    \begin{itemize}
        \item Phase 1 (\texttt{P1\_BasicTools}, 7 Turns): Active Arms: \{\texttt{E3\_Calculator}, \texttt{E3\_UnitConverter}\}.
        \item Phase 2 (\texttt{P2\_SciCalc\_Added}, 10 Turns): Active Arms: \{\texttt{E3\_Calculator}, \texttt{E3\_SciCalculator}, \texttt{E3\_UnitConverter}\}.
        \item Phase 3 (\texttt{P3\_SuperCalc\_Arrives}, 10 Turns): Active Arms: \{\texttt{E3\_SciCalculator}, \texttt{E3\_SuperCalc}\}.
        \item Phase 4 (\texttt{P4\_SuperCalc\_Only}, 8 Turns): Active Arms: \{\texttt{E3\_SuperCalc}, \texttt{E3\_Plotter}\}.
    \end{itemize}
\end{itemize}

\subsubsection{Experiment 4 (mQmA) Configuration Details}
\label{app:exp4_details}
\begin{itemize}
    \item \textbf{Description}: Both queries (randomly drawn from phase-specific sets) and actions change over \(T=28\) turns (total across phases).
    \item \textbf{Master Arm Configurations}:
        \item \texttt{E4\_Translate\_EN\_DE} (German Translator): (\(p^{\text{true}}=0.9, p^{\text{subopt}}=0.1\))
        \item \texttt{E4\_Summarize\_News} (News Summarizer): (\(p^{\text{true}}=0.85, p^{\text{subopt}}=0.15\))
        \item \texttt{E4\_Weather\_API} (City Weather): (\(p^{\text{true}}=0.92, p^{\text{subopt}}=0.1\))
        \item \texttt{E4\_Image\_Resize} (Image Resizer): (\(p^{\text{true}}=0.8, p^{\text{subopt}}=0.05\))
        \item \texttt{E4\_Code\_Python} (Python Code Assistant): (\(p^{\text{true}}=0.75, p^{\text{subopt}}=0.2\))
        \item \texttt{E4\_General\_QA} (Knowledge Bot): (\(p^{\text{true}}=0.7, p^{\text{subopt}}=0.3\))
    \item \textbf{Master Query Configurations}:
        \item \texttt{Q\_Translate\_Hello\_DE} (Optimal: \texttt{E4\_Translate\_EN\_DE})
        \item \texttt{Q\_Summarize\_Article} (Optimal: \texttt{E4\_Summarize\_News})
        \item \texttt{Q\_Weather\_Berlin} (Optimal: \texttt{E4\_Weather\_API})
        \item \texttt{Q\_Resize\_Logo} (Optimal: \texttt{E4\_Image\_Resize})
        \item \texttt{Q\_Python\_Loop} (Optimal: \texttt{E4\_Code\_Python})
        \item \texttt{Q\_Capital\_France} (Optimal: \texttt{E4\_General\_QA})
        \item \texttt{Q\_Weather\_Tokyo} (Optimal: \texttt{E4\_Weather\_API})
        \item \texttt{Q\_Python\_Function} (Optimal: \texttt{E4\_Code\_Python})
    \item \textbf{Phase Details (Total 28 Turns)}:
    \begin{itemize}
        \item Phase 1 (\texttt{P1\_Lang\_Summary}, 8 Turns): Active Arms: \{\texttt{E4\_Translate\_EN\_DE}, \texttt{E4\_Summarize\_News}, \texttt{E4\_General\_QA}\}. Active Queries: \{\texttt{Q\_Translate\_Hello\_DE}, \texttt{Q\_Summarize\_Article}, \texttt{Q\_Capital\_France}\}.
        \item Phase 2 (\texttt{P2\_Weather\_Image}, 10 Turns): Active Arms: \{\texttt{E4\_Weather\_API}, \texttt{E4\_Image\_Resize}, \texttt{E4\_General\_QA}\}. Active Queries: \{\texttt{Q\_Weather\_Berlin}, \texttt{Q\_Resize\_Logo}, \texttt{Q\_Capital\_France}, \texttt{Q\_Weather\_Tokyo}\}.
        \item Phase 3 (\texttt{P3\_Coding\_Focus}, 10 Turns): Active Arms: \{\texttt{E4\_Code\_Python}, \texttt{E4\_General\_QA}, \texttt{E4\_Weather\_API}\}. Active Queries: \{\texttt{Q\_Python\_Loop}, \texttt{Q\_Capital\_France}, \texttt{Q\_Weather\_Tokyo}, \texttt{Q\_Python\_Function}\}.
    \end{itemize}
\end{itemize}

\subsection{Additional plots}
\label{app:reward_plots}
The following figures illustrate the average cumulative regret accrued by the agent under each condition. These trends generally corroborate the findings from the reward analysis.

\begin{figure*}[htbp!]
    \centering
    \includegraphics[width=\textwidth]{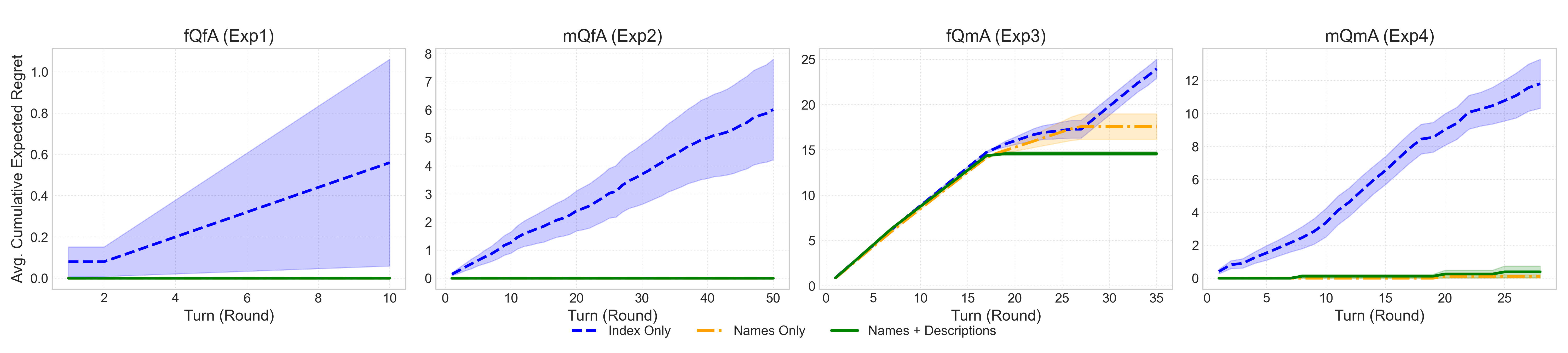} 
    \caption{Average Cumulative Expected Regret across Experiments 1-4. Subplot titles use abbreviations: f=Fixed, m=Moving, Q=Queries, A=Actions. Shaded regions represent \(\pm 1\) standard error of the mean (SEM) across trials. Note the varying x and y-axis scales across subplots, reflecting different experiment durations and regret magnitudes.}
    \label{fig:summary_regret_plots}
\end{figure*}

The experimental results, summarized by the average cumulative expected regret curves in figure \ref{fig:summary_regret_plots}, consistently demonstrate the profound impact of semantic context on the LLM's in-context learning and adaptation for tool selection.

\textbf{Static Environments (Exp1: fQfA; Exp2: mQfA):}
In environments with fixed action spaces and query distributions, the provision of rich semantic information via \textbf{Names + Descriptions (ND)} yields unequivocally superior performance. As illustrated in \ref{fig:summary_regret_plots} (Exp1 and Exp2 panels), the ND condition (green solid line) maintains a cumulative expected regret near zero throughout. This indicates that detailed tool descriptions enable the LLM to rapidly and accurately identify the optimal tool for a given query from the initial turn, effectively bypassing the need for substantial exploration. The LLM, in this condition, behaves as if endowed with strong priors that align well with the task structure.

In stark contrast, the \textbf{Index Only (IO)} condition (blue dashed line) results in the highest cumulative regret, which increases approximately linearly. This suggests that in the absence of semantic anchors, the LLM struggles to discern effective query-action mappings, leading to inefficient, near-random exploration or persistent suboptimal choices. The \textbf{Names Only (NO)} condition (orange dash-dot line) performs comparably poorly to IO in these static settings, indicating that simple tool names alone provide insufficient semantic grounding for the LLM to reliably infer optimal behavior or differentiate tool efficacies.

\textbf{Dynamic Environments (Exp3: fQmA; Exp4: mQmA):}
Non-stationary environments, characterized by changes in the available toolset and/or query distribution, reveal more nuanced interactions between semantic context and adaptability.

In Experiment 3 (fQmA: fixed query, moving actions), the ND condition again demonstrates robust adaptation (\ref{fig:summary_regret_plots}, Exp3 panel). While regret initially accumulates for all conditions due to the unavailability of the globally optimal tool (``E3\_SuperCalc''), the ND agent's regret plateaus sharply around turn 17. This event corresponds to a phase change introducing ``E3\_SuperCalc'' (details in \ref{app:exp3_details}), which the ND agent immediately identifies and exploits due to its descriptive alignment with the complex, fixed query. The NO agent also adapts, but with a discernible delay and to a higher regret plateau, suggesting that names offer some utility for identifying new tools but are less efficient than full descriptions. The IO agent fails to adapt effectively, its regret continuing to climb.

Experiment 4 (mQmA: moving queries and actions) presents the most complex scenario and yields a particularly insightful result (\ref{fig:summary_regret_plots}, Exp4 panel). Here, while the ND condition maintains very low regret overall, demonstrating strong adaptability, the \textbf{Names Only (NO) condition performs remarkably well, achieving comparable, and at times even slightly lower, cumulative regret than ND}, particularly in later phases. This unexpected outcome suggests that in highly dynamic and complex environments, excessively detailed or numerous descriptions might introduce a cognitive load or risk of misinterpretation that slightly hampers agility. Concise, sufficiently indicative names, under such rapid flux, might allow the LLM to maintain a more flexible mapping, leading to surprisingly robust performance. Nevertheless, both NO and ND vastly outperform the IO condition, which struggles profoundly with the compound dynamism.

\subsection{Example Reasoning Traces}
\label{app:subsec_icl_reasoning_traces}

Consider, for example, responses to the query:

\begin{quote}   
\textit{ "query\_text": "I have a list of sales figures for the last quarter, can you help me understand the growth pattern?",
}
\end{quote}

\begin{itemize}
    \item Index only:
\begin{itemize}
    \item "Reasoning: The user is asking for help understanding sales figures and growth patterns. This suggests a need for data analysis and visualization. Without knowing the specific functionality of each action, I will assume Action 1 is a general-purpose tool that can handle data analysis tasks. Actions 2 and 3 might be specialized tools that are less likely to be relevant to this general request.\textbackslash n\textbackslash n Chosen Action: Action 1\textbackslash n",
    \item  "Reasoning: The query is about sales figures and understanding growth patterns. Action 1 has consistently provided rewards for similar queries in the past. Therefore, it is likely the best choice for this query as well.\textbackslash n Chosen Action: Action 1",
\end{itemize}
\item Names:
\begin{itemize}
    \item "Reasoning: The query asks for help understanding a growth pattern in sales figures. This clearly points to data analysis as the appropriate action. The other actions, Text Formatter and Image Resizer, are not relevant to numerical data analysis.\textbackslash n\textbackslash n Chosen Action: Data Analyzer",
    \item "Reasoning: The query is about understanding the growth pattern of sales figures. The "Data Analyzer" tool has consistently provided a reward of 1 for similar queries in the past. The other tools, "Text Formatter" and \"Image Resizer\", are not relevant to analyzing numerical data. Therefore, the "Data Analyzer" is the most likely action to yield a reward.\textbackslash n\textbackslash n Chosen Action: Data Analyzer",
\end{itemize}
\item Names + Descriptions:
\begin{itemize}
    \item "Reasoning: The query mentions \"sales figures\" and \"growth pattern\", which indicates numerical data and the need to find trends. The Data Analyzer tool is specifically designed for processing numerical data arrays to find trends. The other tools, Text Formatter and Image Resizer, are not relevant to this query. Therefore, the Data Analyzer is the most likely to yield a reward.\textbackslash n\textbackslash n Chosen Action: Data Analyzer",
    \item  "Reasoning: The query explicitly mentions \"sales figures\" and \"growth pattern,\" which are numerical data-related tasks. The Data Analyzer is designed to process numerical data arrays to find trends. In the previous turn with a similar query, the Data Analyzer yielded a reward of 0. However, this could be due to chance, and the Data Analyzer is still the most appropriate tool for this type of query. The other tools are clearly not relevant.\textbackslash n\textbackslash n Chosen Action: Data Analyzer",
\end{itemize}

\end{itemize}

\subsection{Filter Reason Act Pipeline}
\label{app:alg_fireact}

We use the following prompt for 0-shot experiment:

\begin{promptbox}
    f"""[BEGIN OF TASK INSTRUCTION]
You are an expert in composing functions. You are given a question and a set of possible functions.
Based on the question, you will need to make one or more function/tool calls to achieve the purpose.
If none of the function can be used, point it out and refuse to answer.
If the given question lacks the parameters required by the function, also point it out.
[END OF TASK INSTRUCTION]
[BEGIN OF AVAILABLE TOOLS]
{actions_prompt_part}
[END OF AVAILABLE TOOLS]
[BEGIN OF FORMAT INSTRUCTION]
The output MUST strictly adhere to the following JSON format,
and NO other text MUST be included.
The example format is as follows. Please make sure the
parameter type is correct. If no function call is needed,
please make tool_calls an empty list ’[]’
‘‘‘
{{
"tool_calls": [
{{"name": "func_name1", "arguments": {{"argument1": "value1", "argument2": "value2"}}}},
... (more tool calls as required)
]
}}
‘‘‘
[END OF FORMAT INSTRUCTION]
[BEGIN OF QUERY]
User Query: {query}
[END OF QUERY]
"""
\end{promptbox}

where \verb|actions_prompt_part| are the available actions with descriptions in the respective IO, NO, DO or DN format and \verb|query| is the respective task. 

The LLM as judge model used was \texttt{gemini-2.5-flash-light}.

\end{document}